\documentclass[11pt]{article}
	
	\newcommand{\blind}{0}
	
    \makeatletter
    \renewcommand\section{\@startsection {section}{1}{\z@}%
                                       {-3.5ex \@plus -1ex \@minus -.2ex}%
                                       {2.3ex \@plus.2ex}%
                                       {\normalfont\fontfamily{phv}\fontsize{16}{19}\bfseries}}
    \renewcommand\subsection{\@startsection{subsection}{2}{\z@}%
                                         {-3.25ex\@plus -1ex \@minus -.2ex}%
                                         {1.5ex \@plus .2ex}%
                                         {\normalfont\fontfamily{phv}\fontsize{14}{17}\bfseries}}
    \renewcommand\subsubsection{\@startsection{subsubsection}{3}{\z@}%
                                        {-3.25ex\@plus -1ex \@minus -.2ex}%
                                         {1.5ex \@plus .2ex}%
                                         {\normalfont\normalsize\fontfamily{phv}\fontsize{14}{17}\selectfont}}
    \makeatother
\usepackage{booktabs} % For formal tables
\usepackage[ruled,algo2e]{algorithm2e} % For algorithms

\SetAlFnt{\small}
\SetAlCapFnt{\small}
\SetAlCapNameFnt{\small}
\SetAlCapHSkip{0pt}
\IncMargin{-\parindent}

\usepackage{xurl}
\usepackage{graphicx}
\usepackage{amsmath}
\usepackage{amsfonts}
\usepackage{amsthm}
\usepackage{setspace}
\usepackage{appendix}
\usepackage{algorithmicx}
\usepackage{algorithm}
\usepackage{algpseudocode} 
\usepackage{subfigure}
\usepackage[showdeletions]{color-edits}
\usepackage{caption}
\usepackage{paralist}
\usepackage{tikz-network}
\usepackage[hidelinks]{hyperref}
\usepackage[noabbrev, capitalise]{cleveref}
\usepackage{natbib}
\usepackage{latex-macros}
\usepackage{enumitem}
\usepackage{dsfont}
\usepackage[utf8]{inputenc}
\usepackage[T1]{fontenc}

\usepackage{booktabs}

\addauthor{mp}{red}
\addauthor{ar}{blue}

\newcommand{\one}{\mathds 1}

\newcommand{\linktoproof}[1]{\begin{center} \hyperref[#1]{\texttt{[Link to Proof]}}\end{center}}
\newcommand{\linktostatement}[1]{\begin{center} \hyperref[#1]{\texttt{[Link to Statement]}}\end{center}}
\newcommand{\ev}[2]{\mathbb E_{#1}\left [ #2 \right ]}
\renewcommand{\var}[2]{\mathbb V_{#1} \left [ #2 \right ]}

\newcommand{\ppar}[1]{\noindent \textbf { #1 }}

\renewcommand{\Pr}{\mathbb P}
\newcommand{\ovec}[1]{\overline {\vec #1}}
\renewcommand{\hat}[1]{\widehat {#1}}
\def\eps{{\varepsilon}}
\renewcommand{\vec}{\bm}

\newtheorem*{observation*}{Observation}
\newtheorem*{definition*}{Definition}

\newenvironment{proofsketch}{\noindent \emph{Proof Sketch.}}{\hfill$\qed$\medskip}

\begin{document}

\def\spacingset#1{\renewcommand{\baselinestretch}%
			{#1}\small\normalsize} \spacingset{1}
		%%%%%%%%%%%%%%%%%%%%%%%%%%%%%%%%%%%%%%%%%%%%%%%%%%%%%%%%%%%%%%%%%%%%%%%%%%%%%%
		
		\if0\blind
		{
			\title{\bf Differentially Private Distributed Estimation and Learning}
			\author{Marios Papachristou$^a$ and M. Amin Rahimian$^b$ \\
			$^a$Cornell University, \texttt{papachristoumarios@cs.cornell.edu}\\
             $^b$University of Pittsburgh, \texttt{rahimian@pitt.edu}}
			\date{}
			\maketitle
		} \fi
		
		\if1\blind
		{

            \title{\bf \large{Differentially Private Distributed Estimation and Learning}}
			\author{Author information is purposely removed for double-blind review}
			
% \bigskip
			% \bigskip
			% \bigskip
			\begin{center}
				{\Large\bf Differentially Private Distributed Estimation and Learning}
			\end{center}
			% \medskip
		} \fi
		% \bigskip
		
	\begin{abstract}

We study distributed estimation and learning problems in a networked environment where agents exchange information to estimate unknown statistical properties of random variables from their privately observed samples. The agents can collectively estimate the unknown quantities by exchanging information about their private observations, but they also face privacy risks.    Our novel algorithms extend the existing distributed estimation literature and enable the participating agents to estimate a complete sufficient statistic from private signals acquired offline or online over time and to preserve the privacy of their signals and network neighborhoods. This is achieved through linear aggregation schemes with adjusted randomization schemes that add noise to the exchanged estimates subject to differential privacy (DP) constraints, both in an offline and online manner. We provide convergence rate analysis and tight finite-time convergence bounds. We show that the noise that minimizes the convergence time to the best estimates is the Laplace noise, with parameters corresponding to each agent's sensitivity to their signal and network characteristics. Our algorithms are amenable to dynamic topologies and balancing privacy and accuracy trade-offs. Finally, to supplement and validate our theoretical results, we run experiments on real-world data from the US Power Grid Network and electric consumption data from German Households to estimate the average power consumption of power stations and households under all privacy regimes and show that our method outperforms existing first-order privacy-aware distributed optimization methods.

	\end{abstract}
			
	\noindent%
	\spacingset{1.45}

\section{Introduction} \label{sec:introduction}

Differential privacy (DP) is a gold standard in privacy-preserving algorithm design that limits what an adversary (or any observer) can learn about the inputs to an algorithm by observing its outputs \citep{dwork2011firm,dwork2014algorithmic}, according to a privacy budget that is usually denoted by $\eps$. It requires that given the output, the probability that any pair of adjacent inputs generate the observed output should be virtually the same. Adding noise to the input data helps enforce this standard in different settings --- e.g., for distributed learning --- but the added noise can also degrade our performance, e.g., lowering the quality of distributed estimation and collective learning for which agents exchange information. 

This paper provides aggregation algorithms that facilitate distributed estimate and learning among networked agents while accommodating their privacy needs (e.g., protecting their private or private signals and network neighborhoods). Each algorithm implies a different tradeoff between its quality of collective learning and how much privacy protection it affords the participating agents (i.e., their privacy budgets). Our performance metrics reflect how distributional features of the private signals and the nature of privacy needs for individual agents determine the learning quality and requisite noise.

Decentralized decision-making and distributed learning problems arise naturally in a variety of applications ranging from sensor and robotic networks in precision agriculture, digital health, and military operations to the Internet of things and social networks \citep{bullo2009distributed,jackson2008}; see~\cref{sec:related_work} for a detailed literature review. We are particularly interested in distributed estimation problems that arise in smart grids with distributed generation and energy resources. Notably, a recent report from \cite{nas,national2023role} suggests that net metering practices should be revised to reflect the value of distributed electricity generation, such as rooftop solar panels. Net metering compensates customers for the electricity they provide to the grid through distributed generation. The report notes that net metering has facilitated the embrace of distributed generation in states where it has been put into effect, resulting in levels surpassing 10\% in a few states and projected to rise in both these and other states. Additionally, the report emphasizes the need to revisit and evolve net metering policies to support the deployment of distributed generation that adds value in reducing fossil fuel use, enhancing resilience, and improving equity. In this context, each customer faces an individual privacy risk in sharing their estimates since revealing exact measurements can pose security risks that can be leveraged by an adversary (e.g., understanding when someone is at their home, daily habits, family illness, etc.), and, therefore, developing privacy-preserving methods that support decentralized decision making in such setups is critical.

This paper introduces novel algorithms designed for the distributed estimation of the expected value of sufficient statistics in an exponential family distribution. The proposed methods leverage signals received by individual agents, who maintain and update estimates based on both these signals and information from their local neighborhood. Our contributions to the existing literature on distributed optimization include new privacy-aware distributed estimation algorithms that exhibit faster convergence rates compared to established first-order methods (cf. \cite{rizk2023enforcing}). Notably, our algorithms safeguard the information in agents' signals and local neighborhood estimates, ensuring optimal convergence times to true estimates. Furthermore, in contrast to existing approaches, our algorithms can support privacy-aware estimation within an online learning framework, accommodate dynamic topologies, and balance privacy and accuracy by distributing the privacy budget among agents. Finally, we verify our proposed algorithms on real-world datasets and show that they outperform existing first-order methods. 

\subsection{Main Results} \label{sec:contribution}

\noindent{\bf Summary of main notations.} We use barred bold letters to denote vectors (e.g., $\ovec x$) and bold letters to indicate vector components (e.g., $\vec x_i$ is a component of $\ovec x$). We use capital letters to denote matrices (e.g., $A$) and small letters to represent their entries (e.g., $a_{ij}$). We use small letters to denote scalars. We denote the $n\times 1$ column vector of all ones by $\mathds 1$.

\noindent{\bf Summary of the Problem Setup (see \cref{sec:model_dynamics}).} We consider a network of $n$ agents indexed by $[n] = \{1,\ldots,n\}$ whose interconnections are characterized by a symmetric, doubly-stochastic, adjacency matrix $A$. This adjacency structure, encoded by graph neighborhoods, $\cN_i = \{j:a_{ij} \neq 0\}, i\in[n]$, may be a consequence of geospatial constraints such as sensing and communication range or geographic proximity; it can also be a reflection of how the network has evolved and other engineering considerations, e.g., which nodes belong to which countries or companies in a multi-party network. The adjacency weights may also result from geoeconomic constraints such as access to local trade and business intelligence (contracts, sales, and filled orders). In the case of social networks, they can also represent the presence of influence and mutual interactions among individuals.

Given the adjacency structure $A$, at every round $t \in \{1,2 , 3, \ldots\}$, each agent $i$ receives a private signal $\vec s_{i, t}$ that is sampled from an exponential family distribution with the natural sufficient statistic $\xi(\cdot) \in \Rbb$ and a common, unknown parameter $\theta$ belonging to a measurable set $\Theta$. The goal of the agents is to collectively estimate the common value of $\ev {\theta}{\xi(\vec s)}$ by combining their samples. This is achieved through a consensus algorithm by forming an estimate $\vec \nu_{i, t}$ and exchanging it with their network neighbors in a distributed manner respecting the adjacency structure of $A$. The agents also want to control the information that is leaked about their signals and the estimates of their neighbors $ \{ \vec \nu_{j, t - 1} \}_{j \in \cN_i}$. They add noise $\vec d_{i, t}$ to their updates. The noise level should be high enough not to violate an $\eps$ differential privacy budget. Briefly, we say that a mechanism $\cM$ is $\eps$-DP when for any pair of ``adjacent'' inputs (i.e., private signals or private signals and neighboring estimates), the logarithm of the probability ratio of any output in the range space of $\cM$ being resulted from either of the adjacent inputs is bounded by $\eps$:  $|\log({\Pr[\cM(\vec s) \in R]}/{\Pr [\cM(\vec s') \in R]})| \le \eps$, for all adjacent pairs ($\vec s$ and $\vec s'$) in the input domain and any subset $R$ of the output range space. Our specific notion of adjacency between the input pairs will be determined by the nature of information leaks, i.e., private signals or private signals and network neighborhoods, against which the exchanged estimates ($\vec \nu_{i, t}$) are being protected (\cref{fig:dp_protections}).  

\begin{figure}[htb]
    \centering
    \subfigure[\footnotesize Signal DP]{\includegraphics[width=0.31\textwidth]{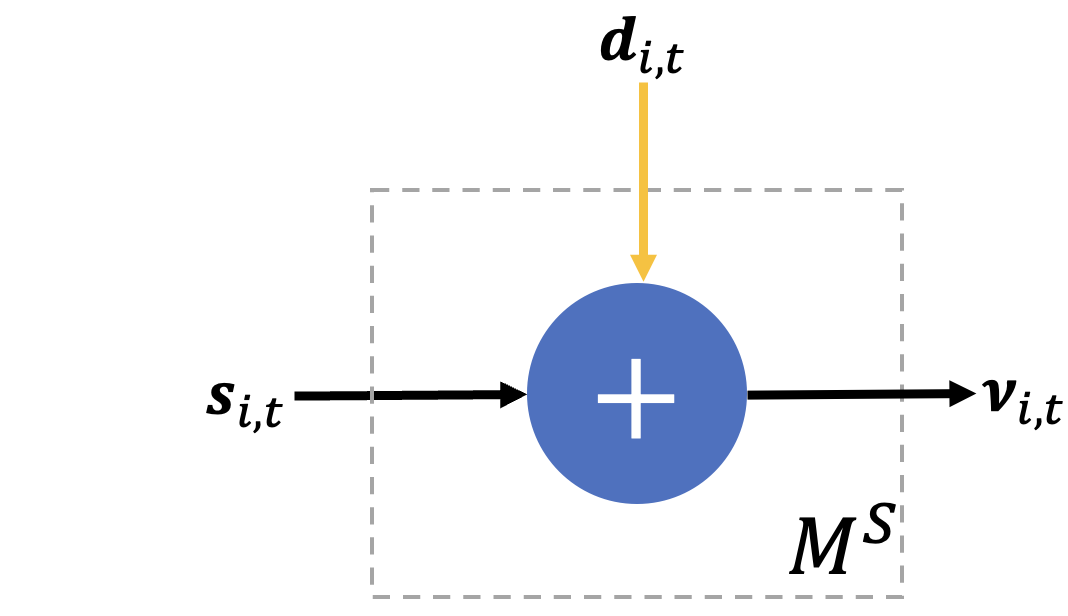}}
    \subfigure[\footnotesize Network DP]{\includegraphics[width=0.31\textwidth]{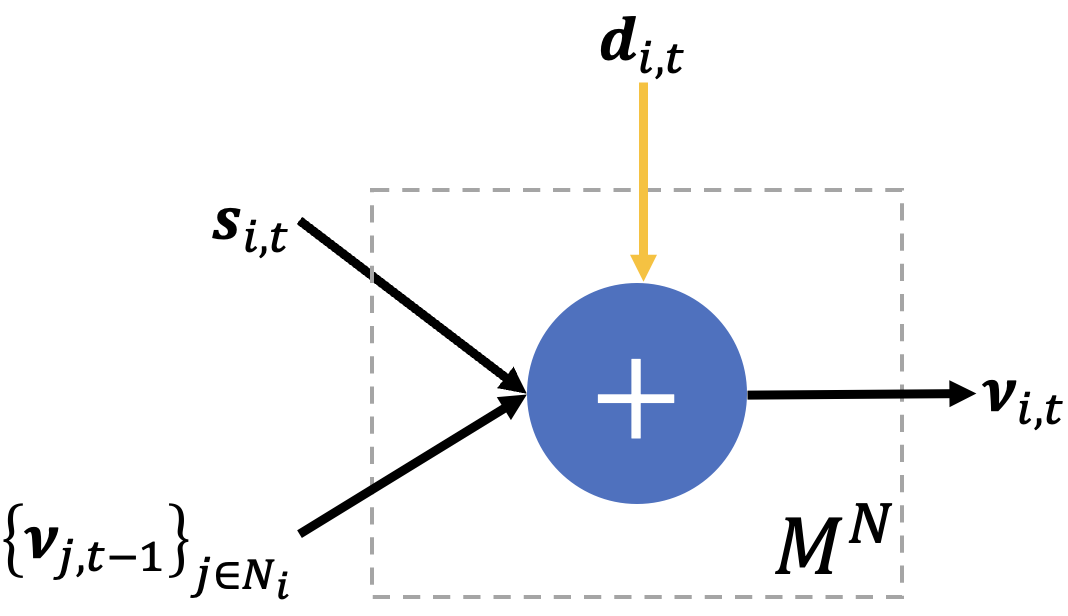}}
    \subfigure[\footnotesize Performance Metrics]{\includegraphics[width=0.27\textwidth]{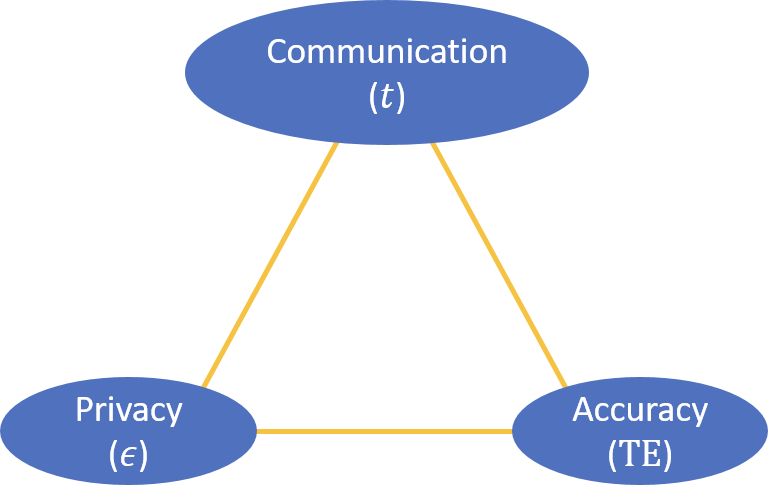}}
    \caption{Two types of DP protections considered in this paper are signal DP, $\cM^{S}$, and network DP, $\cM^{N}$. The private signal of agent $i$ at round $t$ is denoted by $\vec s_{i, t}$, $\vec d_{i, t}$ corresponds to the noise added from agent $i$ at round $t$, and $\vec \nu_{i, t}$ corresponds to the estimate of agent $i$ at round $t$. Our theoretical guarantees delineate the relationship between communication resources ($t$ rounds), privacy budget ($\eps$-DP), and total error (TE). Signal and network DP imply 
 different performance tradeoffs as detailed in \cref{tab:detailed_bounds}.}
    \label{fig:dp_protections}
\end{figure}

\ppar{Theoretical Contributions (\cref{sec:min_var_unbiased_estimation_signal,sec:online_learning_dp}).} In this paper, we provide bounds for the convergence of the DP estimates $\ovec \nu_t = (\nu_{i, t})_{i\in[n]}$ to the desired value $\ovec m_{\theta} := \mathds 1\vec m_{\theta}$, where  $\vec m_{\theta} = \ev {\theta} {\xi(\vec s)}$. We decompose the total error as follows: 
\begin{align*}
    \underbrace{\ev {} {\left \| \ovec \nu_t - \ovec m_\theta \right \|_2}}_{\text{Total Error (TE)}} \le \underbrace {\ev {} {\left \| \ovec \nu_t - \ovec \mu_t \right \|_2}}_{\text{Cost of Privacy (CoP)}} + \underbrace {\ev {} {\left \| \ovec \mu_t - \ovec m_\theta \right \|_2}}_{\text{Cost of Decentralization (CoD)}}
\end{align*} Here, $\ovec \mu_t= (\vec \mu_{i, t})_{i\in[n]}$ corresponds to the vector of non-private estimates.  The first term corresponds to the ``Cost of Privacy'' (CoP), the estimation cost incurred by the $\eps$-DP noising. The second term corresponds to the expected error from running the non-private distributed learning algorithm and, therefore, measures the ``Cost of Decentralization'' (CoD). In \cref{sec:MVUE}, when we consider ``offline'' estimation of $\vec m_{\theta}$ from a fixed collection of initial signals available at the beginning ($t=0$), we replace $\vec m_{\theta}$ by the best possible estimate --- the minimum variance unbiased estimate (MVUE) --- of an omniscient observer who has \textit{centralized} access to all the private signals. 

Our goal is to find differentially private aggregation mechanisms $\cM$ with fast convergence guarantees that minimize CoP parametrized by noise distributions $\left \{ \cD_{i, t} \right \}_{i \in [n], t \ge 1}$ of the agents and subject to their $\eps$ differential privacy budget constraints, i.e., ``s.t. $\eps$-DP'':
\begin{align}
    \mathrm{CoP}(\cM) = \inf_{\left \{ \cD_{i, \tau} 
 \right \}_{i \in [n], \tau \in [t]} \text{ s.t. $\eps$-DP}} \ev {\left \{ (\ovec \nu_\tau, \ovec s_{\tau}) \right \}_{\tau \in [t]}} {\left \| \ovec \nu_t - \ovec \mu_t \right \|_2}.
\end{align}

From now on, we will refer to this optimal value as the cost of privacy. In our analysis, the CoP and CoD are proportional to the signal and noise variance. The convergence rate also depends on the number of nodes $n$ and the \emph{spectral gap} $\beta^\star$ of the doubly-stochastic adjacency matrix $A$, which dictates the convergence rate of $A^t$ to its limiting matrix $(\mathds 1 \mathds 1^T)/n$, as $t\to\infty$.  

Subsequently, the noise distribution can be optimized by minimizing the weighted variances of the noise terms in the upper bounds subject to DP constraints. Our main results are summarized in \cref{tab:detailed_bounds}. These consist of minimum variance unbiased estimation (\cref{sec:MVUE}) and online learning (\cref{sec:online_learning_dp}) of expected values under \emph{(i)} protection of the private signals (Signal DP), and \emph{(ii)} protection of the private signals and the local neighborhoods (Network DP). The following Informal Theorem summarizes our theoretical contributions along with \cref{tab:detailed_bounds}:

\vspace{-0.5\baselineskip}

\begin{quote}
% \begin{informaltheorem}
    \emph{\textbf{Informal Theorem.} If $\Delta = \max_{\vec s \in \cS} \left | \frac {d \xi(\vec s_j)} {d \vec s_j} \right |$ is the global sensitivity of $\xi$ at time $t$, and $[a_{ij}]_{i,j = 1}^n$ are the weights of the adjacency matrix, then the upper bound in CoP is minimized by the Laplace noise with parameters $\frac {\Delta} {\eps}$ for the case of Signal DP, and parameters $\frac {\max \{ \max_{j \in \cN_i} a_{ij}, \Delta \}} {\eps}$ for the case of network DP.} 
% \end{informaltheorem}
\end{quote}

Moreover, whenever the global sensitivity $\Delta$ is unbounded, something that happens to a variety of sufficient statistics $\xi$, we rely on the smooth sensitivity introduced by \cite{nissim2007smooth} and derive the same algorithms (but with parameters depending on the smooth sensitivity instead of $\Delta$), which achieve $(\eps, \delta)$-DP, namely using the smooth sensitivity introduces a compromise in the $\eps$-DP guarantee by a small information leakage probability $\delta$, relaxing the DP constraint for a mechanism $\cM$ as follows: $\Pr[\cM(\vec s) \in R] \leq e^{\eps}{\Pr [\cM(\vec s') \in R]} + \delta$, for all adjacent input pairs ($\vec s$ and $\vec s'$) and all subsets $R$ of the output range space.

\begin{table}[H]
    \centering
    \caption{Total Error Bounds. $\Delta$ is the maximum signal sensitivity (absolute value of the derivative of $\xi(\cdot)$), ${M}_n$ $ =$ $ \max\limits_{i\in[n]}|\xi(\vec{s}_i)|$, $a = \max_{i \neq j} a_{ij}$ is the maximum non-diagonal entry of the adjacency matrix $A$, and $\beta^\star = \max \{ \lambda_2(A), |\lambda_n(A)| \}$ is the spectral gap of $A
    $. The {\color{blue}{blue}} terms are due to privacy constraints (CoP), and the {\color{red}{red}} terms are due to decentralization (CoD).}
    \footnotesize
    \begin{tabular}{lcc}
        \toprule
            &  Minimum Variance Unbiased Estimation & Online Learning of Expected Values \\
        \toprule
                Signal DP & \cref{theorem:cop_min_var_unbiased_estimation_signal}%alg:min_var_unbiased_estimation_signal 
                & \cref{theorem:cop_online_learning_dp_signal}%alg:online_learning_dp_signal 
                \\
         & $O \left ( n (\beta^\star)^t\left ( {\color{blue}{\frac {\Delta} {\eps}}} + {\color{red}{M_n}} \right )  + {\color{blue}{\frac {\Delta} {\eps}}} \right )$ & $O \left (  \frac {n} {\sqrt t}   \left ( {\color{blue}{\frac {\Delta} {\eps}}} + {\color{red}{\sqrt {\var {} {\xi(\vec s)}}}} \right ) \right )$ \\
         \midrule
         Network DP & \cref{corollary:cop_min_var_unbiased_estimation_network}%alg:min_var_unbiased_estimation_signal 
         & \cref{theorem:cop_online_learning_dp_network}%alg:online_learning_dp_network 
         \\ 
         & $O \left ( n (\beta^\star)^t\left ( {\color{blue}{\frac {\max \{a, \Delta \}} {\eps}}} + {\color{red}{M_n}} \right )  + {\color{blue}{\frac {\max \{ a, \Delta\}} {\eps}}} \right )$ & $O \left ( \frac {n} {\sqrt t} \left ( {\color{blue}{\frac {\max \{a, \Delta \}} {\eps}}} + {\color{red}{\sqrt {\var {} {\xi(\vec s)}}}} \right ) \right )$ \\ 
        
        \bottomrule
    \end{tabular}
    
    \label{tab:detailed_bounds}
\end{table}

\ppar{Experiments (\cref{sec:experiments}).} We conduct experiments with two real-world datasets motivated by decentralized decision problems in power grids (see \cref{sec:introduction}). The first dataset considers the daily consumption of several German Households over three years \citep{milojkovic2018gem}, and the second one considers the US Power Grid network from \cite{watts1998collective}. Our experiments show that we can achieve $(\eps, \delta)$-DP while not significantly sacrificing convergence compared to the non-private baselines. The results also indicate the increased challenges in ensuring network DP compared to Signal DP and the importance of distributional features of the signals, in particular, having sufficient statistics with bounded derivatives. 

\ppar{Code and Data.} For reproducibility, we supplement the manuscript with our code and data, which can be found at: \url{https://github.com/papachristoumarios/dp-distributed-estimation}. 

\subsection{Related Work} \label{sec:related_work}

Our results relate to different bodies of literature across the engineering, statistics, and economics disciplines, and in what follows, we shall expand upon these relations.

\ppar{Decentralized Decision Making and Distributed Learning} have attracted a large body of literature over the years, with notable examples of \cite{BorkarVaraiya82,tsitsiklis1984convergence,tsitsiklis1993decentralized,papachristou2024group}. Recently, there has been a renewed interest in this topic due to its applications to sensor and robotic networks \citep{chamberland2003decentralized,olfati05,kar2012distributed,atanasov2014distributed,7040469} and emergence of a new literature considering networks of sensor and computational units  \citep{shahrampour2015distributed,nedic2015nonasymptotic,nedic2015fast}. Other relevant results investigate the formation and evolution of estimates in social networks and subsequent shaping of the individual and mass behavior through social learning \citep{krishnamurthy2013social,7081738, 6874893,rahimian2016bayesian,rahimian2016group}. Obtaining a global consensus by combining noisy and unreliable locally sensed data is a crucial step in many wireless sensor network applications; subsequently, many sensor fusion schemes offer good recipes to address this requirement \citep{xiao2005scheme,xiao2006space}. In many such applications, each sensor estimates the field using its local measurements, and then, the sensors initiate distributed optimization to fuse their local forecast. If all the data from every sensor in the network can be collected in a fusion center, then a jointly optimal decision is readily available by solving the global optimization problem given all the data \citep{alexandru2020towards}. However, many practical considerations limit the applicability of such a centralized solution. This gives rise to the distributed sensing problems that include distributed network consensus or agreement \citep{aumann1976agreeing,geanakoplos1982we, BorkarVaraiya82}, and distributed averaging \citep{dimakis2008geographic}; with close relations to the consensus and coordination problems that are studied in the distributed control theory \citep{jadbabaie2003coordination,mesbahiBook,bullo2009distributed}. Our work contributes to this body of work by providing a privacy-aware method for distributed estimation over a network. 

The work most importantly connected to our work is the work of \cite{rizk2023enforcing}, which introduces a first-order DP method for distributed optimization over a network of agents. While their work presents a general first-order method for DP distributed optimization, which is suitable for a larger family of optimization problems that include MVUE, adapting their method to our task comes with important trade-offs in the quality of the estimation, as running their method results in significantly higher error (up to $1000\times$ more; see \cref{fig:mse_plot}) for the MVUE task due to the repeated inclusion of the signal in the belief updates. Moreover, contrary to ours, their method does not support the online learning regime. Finally, the concept of ``graph-homomorphic noise'' proposed in their paper is equivalent to the Network DP regime of our paper. 

\ppar{Differential privacy} is a modern definition of data privacy encompassing many previous definitions, such as $K$-anonymity. A mechanism is differentially private if it maps similar records to the same value with equal probability. One consequence of the definition is that it guarantees that the outcome of a statistical analysis will be identical whether the individual chooses to participate in the social learning process. Many previously proposed mechanisms can be shown to be differentially private, such as randomized response \citep{warner1965randomized}, the Laplace mechanism, and the Gaussian mechanism. 
The randomized response algorithm originally proposed by \citet{warner1965randomized} consists of randomly perturbing binary responses, and it allows the population means to be recovered while giving individual respondents plausible deniability -- an instance of adding noise to data.

Statistical disclosure control of donated data, e.g., submitting recommendations to a public policy agency or submitting online product reviews to an e-commerce platform, requires safeguarding donors' privacy against adversarial attacks where standard anonymization techniques are shown to be vulnerable to various kinds of identification \citep{barbaro2006face}, linkage and cross-referencing  \citep{sweeney2015only,proc,sweeney1997weaving}, and statistical difference and re-identification attacks \citep{kumar2007anonymizing}. Here, we propose to analyze the efficiency of distributed estimation where agents learn from each other's actions protected by the gold standard of differential piracy \citep{dwork2011firm} to optimize statistical precision against the privacy risks to data donors who engage in social learning. 

DP can be implemented using central or local schemes. In centralized DP implementations, individual data are collected, and privacy noise is added to data aggregates, used, e.g., in the U.S. Census Bureau's Disclosure Avoidance system \citep{censusDP}. In local implementations, DP noising is done as data is being collected. Local methods forego the need for any trusted parties and typically provide more fundamental protection that can withstand a broader range future infiltration, e.g., even a government subpoena for data cannot violate the privacy protection when collected data is itself subject to privacy noising  --- e.g., Google, LinkedIn, and Apple's DP noising of their user data \citep{appleDP,googleDP,cardoso2022differentially,rogers2020linkedin}, cf. \citep{wilson2020differentially,googleDP2}. 

Regarding privacy and networks, \cite{koufogiannis2017diffusing} present a privacy-preserving mechanism to enable private data to diffuse over social networks, where a user wants to access another user's data and provide privacy guarantees on the privacy leak, which depends on the shortest path between two users in the network.  \cite{alexandru2021private} study the problem of private weighted sum aggregation with secret weights, where a central authority wants to compute the weighted sum of the local data of some agents under multiple privacy regimes. \cite{rahimian2023differentially,rahimian2023seeding} study influence maximization using samples of network nodes that are collected in a DP manner. 

Our work contributes to the above line of work by introducing a novel DP mechanism for distributed estimation and learning of exponential family distributions. Particularly, to the best of our knowledge, in the online learning regime, our algorithm introduces a novel weighting scheme that can protect both the individual signals and the neighboring beliefs, which can efficiently learn the expected value of the sufficient statistic. Moreover, we also provably derive the optimal distributions that minimize the convergence time of the algorithm and show that they are the Laplace distributions with appropriately chosen parameters. 

\ppar{Cyber-Physical Systems} (e.g., energy, transportation systems, healthcare systems, etc.) correspond to the building blocks of modern information and communication technologies, whose privacy and security are crucial for the function of such technologies. There have been multiple methods, such as encryption and $K$-anonymity, to achieve privacy and security in cyber-physical systems \citep{hassan2019differential,zhang2016privacy,kontar2021internet}. By incorporating differential privacy techniques, such as noise injection or data aggregation, into the design and operation of cyber-physical systems, privacy risks can be mitigated while preserving the utility of the collected data. This ensures that individual privacy is protected, as the data released from these systems cannot be used to infer sensitive information about specific individuals. Moreover, the application of differential privacy to cyber-physical systems enables the collection and analysis of data at scale, allowing for improved system performance, anomaly detection, and predictive maintenance while maintaining the trust of individuals and protecting their privacy in an increasingly connected world \citep{li2010data,gowtham2017privacy,xu2017security}. Our method's efficiency, e.g., compared to \cite{rizk2023enforcing}; see \cref{sec:experiments,app:rizk}, makes it suitable for several large-scale data applications.

\ppar{Federated Learning (FL)} is a collaborative learning paradigm for decentralized optimization without the need to collect all data points in a central server for gradient calculations \citep{mcmahan2022federated,kontar2021internet}, with many applications in mind: distributed training of ML models \citep{bonawitz2021federated,shi2023ensemble,yue2022federated}, healthcare \citep{kaissis2020secure}, wireless communications \citep{niknam2020federated}, etc. While more general than the setup we consider here, it suffers from issues in terms of communication and privacy. Existing privacy-preserving FL methods (cf. \cite{rizk2023enforcing})
usually adopt the instance-level differential privacy (DP), which provides a rigorous privacy guarantee but with several bottlenecks \citep{truong2021privacy}. \cite{truex2019hybrid} proposed a privacy-aware FL system that combines DP with secure multiparty computation, which utilizes less noise without sacrificing privacy as the number of federating parties increases and produces models with high accuracy. Other FL methods, such as \cite{zhang2022understanding}, accommodate differentially private updates via incorporating gradient clipping before adding privacy noise to achieve good performance subject to privacy constraints. 

Contrary to most of these methods, which are first-order optimization methods that are suitable to a large variety of losses compared to our method, our zero-order belief updates for the MVUE are simple, more efficient, and have significantly lower error than first-order methods (see \cref{fig:mse_plot} for comparison with \cite{rizk2023enforcing}). Moreover, our method can learn from data that arrive in an online way, whereas methods such as \cite{rizk2023enforcing,zhang2022understanding} are offline. Finally, most of these approaches rely on SGD. In contrast, our method focuses more on the decision-theoretic and statistical problem of estimating the expected value of the sufficient statistics of signals generated by an exponential family distribution.

\section{Differential Privacy Protections in a Distributed Information Environment}\label{sec:preliminaries}

\subsection{The Distributed Information Aggregation Problem Setting}

Let $\Theta$ be any measurable set, and in particular, not necessarily finite. Consider a network of $n$ agents and suppose that each agent $i\in[n]$ observes an i.i.d. samples $\vec s_{i}$ from a common distribution $\ell(\mathord{\cdot}|\theta)$ over a measurable sample space $\mathcal{S}$ (For simplicity in our proofs, we consider the simple case of 1D signals, i.e., $\cS \subseteq \Rbb$. Extending to multi-dimensional signals (i.e., $\cS \subseteq \Rbb^s$) is straightforward and considers the $\ell_\infty$ norm of the partial derivatives.). 
We assume that $\ell(\mathord{\cdot}|\theta)$ belongs to a one-parameter exponential family so that it admits a probability density or mass function that can be expressed as

\begin{align}\label{expfam}
\ell(\vec s | {\theta})  = \tau(\vec s)e^{\alpha(\theta)^{T} \xi(\vec s) - \varkappa(\alpha(\theta))},
\end{align} 

where $\xi(\vec s) \in \mathbb{R}$ is a measurable function acting as a complete sufficient statistic for the i.i.d. random samples $\vec s_{i}$, and $\alpha: \Theta \to \mathbb{R}$ is a mapping from the parameter space $\Theta$ to the real line $\mathbb{R}$, $\tau(\vec s) >0$ is a positive weighting function, and $\varkappa(\alpha) := \ln \int\limits_{s\in\mathcal{S}}\tau(s)e^{\alpha\xi(s)} d s$ is a normalization factor known as the $\log$-partition function. In \eqref{expfam}, $\xi(\mathord{\cdot})$ is a complete sufficient statistic for $\theta$. It is further true that $\sum_{i=1}^{n}\xi(\vec s_i)$ is a complete sufficient statistic given the $n$ i.i.d. signals that the agents have received \cite[Section 1.6.1]{bickel2015mathematical}. In particular, any inferences that involve the unknown parameter $\theta$ based on the observed signals $\ovec s = (\vec s_i)_{i\in[n]}$ can be equivalently performed given $\sum_{i=1}^{n}\xi(\vec s_i)$. The agents aim to estimate the expected value of $\xi(\mathord{\cdot})$: $\vec m_{\theta} = \ev {} {\xi(\vec s_{i})}$, with as little variance as possible. The Lehmann-Scheff\'{e} theory --- cf. \cite[Theorem 7.5.1]{casella2002statistical} --- implies that any function of the complete sufficient statistic that is unbiased for  $\vec m_\theta$ is the almost surely unique minimum variance unbiased estimator of $\vec m_\theta$. In particular, the minimum variance unbiased estimator of $\vec m_{\theta}$ given the initial data sets of all nodes in the network is given by: $\hat {\vec m_{\theta}} = ({1}/{n})\sum_{i=1}^{n} \xi(\vec s_{i})$. 

For concreteness, we can consider a group of $n$ suppliers whose private signals consist of their contracts, sales orders, and fulfillment data. These suppliers would benefit from aggregating their private information to better estimate market conditions captured by the unknown parameter $\theta$, e.g., to predict future demand. However, sharing their private signals would violate the privacy of their customers and clients. In Section \ref{sec:model_dynamics}, we explain how the agents can compute the best (minimum variance) unbiased estimator of $\vec m_{\theta}$ using average consensus algorithms \cite{olshevsky2014linear} that guarantee convergence to the average of the initial values without direct access to each other's private signals.

\subsection{The Information Exchange Model} \label{sec:model_dynamics}

 We consider an undirected network graph and let the undirected network $\cG(\cV = [n], \cE)$ which corresponds to a Markov chain with a doubly-stochastic symmetric adjacency/transition matrix $A = [a_{ij}]_{i,j=1}^{n}$ with the uniform stationary distribution. For instance, such an adjacency matrix $A = [a_{ij}]_{i,j=1}^{n}$ can defined according to the Metropolis-Hastings weights \citep{boyd2004fastest}:  $a_{ij} = 1/\max\{\deg(i),\deg(j)\}$ if $(j,i) \in \mathcal{E}$, and $[A]_{ij} = 0$ otherwise for $i\neq j$; furthermore, $a_{ii} = 1- \sum_{j\neq i} a_{ij}$. This choice of weights leads to a Markov chain where the stationary distribution is the uniform distribution \citep{boyd2004fastest}, and the agents can set these weights locally based on their own and neighboring degrees without the global knowledge of the network structure. For choices of $A$ that yield the fastest mixing Markov chain (but may not be locally adjustable), see \cite{boyd2004fastest}.

\begin{figure}[t]
\noindent
 \begingroup\fboxsep=10pt
\fbox{
\begin{minipage}{0.95\textwidth}
\captionof{algorithm}{Non-Private Distributed Minimum Variance Unbiased Estimation} \label{alg:non_private_min_var_unbiased_estimation_signal}
The agents initialize with: $\vec\mu_{i,0} = \xi(\vec s_{i})$, and in any future time period the agents communicate their values and update them according to the following rule:  
\begin{align}
{\vec\mu}_{i,t} =  a_{ii} \, {\vec\mu}_{i,t-1} + \sum_{j\in\mathcal{N}_i }a_{ij}{{\vec\mu}_{j,t-1}}. \label{eq:estimateUpdateAvgConcensus-non-private} 
\end{align} 
\end{minipage}}
\endgroup
\vspace{-5pt}
\end{figure}

The mechanisms for convergence, in this case, rely on the product of stochastic matrices, similar to the mixing of Markov chains (cf. \cite{levin2009markov22,shahrampour2015distributed}); hence, many available results on mixing rates of Markov chains can be employed to provide finite time grantees after $T$ iteration of the average consensus algorithm for fixed $T$. Such results often rely on the eigenstructure (eigenvalues/eigenvectors) of the communication matrix $A$, and the facts that it is a primitive matrix and its ordered eigenvalues satisfy $-1<\lambda_n(A)\leq\lambda_{n-1}(A)\leq\ldots\leq\lambda_1(A)=1$, as a consequence of the Perron-Frobenius theory \cite[Theorems 1.5 and 1.7]{seneta2006non}. 

Moreover, another mechanism considers learning the expected values in an online way where agents receive signals at every round and then update their estimate by averaging the estimates of their neighbors, their own estimate, and the new signals (see \cref{alg:non_private_online_learning_dp}). 

\begin{figure}
\noindent
 \begingroup\fboxsep=10pt
\fbox{
\begin{minipage}{0.95\textwidth}
\captionof{algorithm}{Non-Private Online Learning of Expected Values} \label{alg:non_private_online_learning_dp}
Initializing $\vec\mu_{i,0}$ arbitrarily, in any future time period $t\geq 1$ the agents observe a signal $\vec s_{i,t}$, communicate their current values $\vec \mu_{i,t-1}$, and update their beliefs to $\vec \mu_{i,t}$, according to the following rule:  
\begin{align}
{\vec\mu}_{i,t} = & \frac{t-1}{t}\left(a_{ii} \, {\vec\mu}_{i,t-1} + \sum_{j\in\mathcal{N}_i }a_{ij}{{\vec\mu}_{j,t-1}}\right) + \frac{1}{t}\xi(\vec s_{i,t}). \label{eq:beliefUpdateAvgConcensusOnline-non-private} 
\end{align} 
\end{minipage}}
\endgroup 
\vspace{-5pt}
\end{figure}

The $1 / t$ discounting provided in the above algorithm enables learning the expected values $\vec m_\theta = \ev {\theta} {\xi(\vec s)}$ asymptotically almost surely with a variance that scales as $O(1/t)$; i.e., linearly in time. As shown in \cite{rahimian2016distributed}, the variance upper bound comprises two terms. The former term considers the rate at which the Markov chain with transition matrix $A$ is mixing and is governed by the spectral gap, i.e., the second largest magnitude of the eigenvalues of $A$. The latter term captures the diminishing variance of the estimates with the increasing number of samples gathered by all the agents in the network.

Now that we have the necessary background in distributed estimation, we present the two DP protection mechanisms that our paper considers: the Signal DP and the Network DP.

% \section{Distributed Learning Under Privacy Constraints} 

\subsection{Differential Privacy Protections}

\subsubsection{Definitions and Mechanisms}

In this paper, we consider two methods for differential privacy and refer to them as \emph{Signal Differential Privacy (Signal DP)} and \emph{Network Differential Privacy (Network DP)}. Both algorithms are local in principle; the agents simply add noise to their estimates to achieve a desired privacy guarantee. Roughly, Signal DP adds noise to protect the signal $\vec s_{i, t}$ of each agent, and Network DP adds noise to protect the signal $\vec s_{i, t}$ of each agent, as well as the estimates $\{ \vec \nu_{j, t - 1} \}_{j \in \cN_i}$ of her neighbors from round $t - 1$. We assume that the non-private network dynamics evolve as
{\begin{align} \label{eq:dynamics}
    \vec \mu_{i, t} = F_{i, t}(\vec \mu_{i, t - 1}) + G_{i, t} \left ( \left \{ \vec \mu_{j, t - 1} \right \}_{j \in \cN_i} \right ) + H_{i, t}(\vec s_{i, t}), 
\end{align}} for each agent $i \in [n]$, and $t \ge 1$, where $F_{i, t} : \Rbb \to \Rbb$, $G_{i, t} : \Rbb^{d_i} \to \Rbb$, and $H_{i, t} : \cS \to \Rbb$ are functions determined by the learning algorithm, and correspond to the information from the agent's own estimate, the information from the neighboring estimates, and the information from the agent's private signal respectively. To achieve differential privacy, each agent adds some amount of noise $\vec d_{i, t}$ drawn from a distribution $\cD_{i, t}$ to their estimate and reports the noisy estimate to their neighbors. The agent can either aim to protect only their private signal -- which we call Signal DP and denote by $\cM^S$--, or protect their network connections and their private signal -- which we call Network DP and denote by $\cM^N$. The noisy dynamics are:{
\begin{align} \label{eq:noisy_dynamics}
    \vec \nu_{i, t} = F_{i, t}(\vec \nu_{i, t - 1}) + \underbrace {G_{i, t} \left ( \left \{ \vec \nu_{j, t - 1} \right \}_{j \in \cN_i} \right ) + \overbrace {H_{i, t}(\vec s_{i, t})}^{\text{\textcolor{magenta}{Signal DP $\cM^S_{i, t}$}}}}_{\text{\textcolor{orange}{Network DP $\cM^N_{i, t}$}}} \; + \; \vec d_{i, t}
\end{align}}
In \cref{eq:noisy_dynamics,fig:dp_protections}, we have outlined the dynamics of the two types of privacy protections. Formally, the two types of mechanisms can also be written as

\begin{align}
    \psi_{\cM^S_{i, t}}(\vec s_{i, t}) & = H_{i, t} (\vec s_{i, t}) + \vec d_{i, t}, \tag{$\cM^S_{i, t}$} \\
    \psi_{\cM^N_{i, t}}\left ( \vec s_{i, t}, \left \{ \vec \nu_{j, t - 1} \right \}_{j \in \cN_i} \right ) & = H_{i, t} (\vec s_{i, t}) + G_{i, t} \left ( \left \{ \vec \nu_{j, t - 1} \right \}_{j \in \cN_i} \right ) +  \vec d_{i, t}, \tag{$\cM^N_{i, t}$}
\end{align}
and the $\eps$-DP requirement is denoted as 
{\footnotesize
\begin{align*}
    \left | \log \left ( \frac {\Pr \left [ \psi_{\cM^S_{i, t}}(\vec s_{i, t}) = x \right ]}  {\Pr \left [ \psi_{\cM^S_{i, t}}(\vec s_{i, t}') = x \right ]} \right ) \right | \le \eps & \text{ for all } \vec s_{i, t}, \vec s_{i, t}' \in \cS \text{ s.t. } \left \| \vec s_{i, t} - \vec s_{i, t}'\right \|_1 \le 1 \\
    \left | \log \left ( \frac {\Pr \left [ \psi_{\cM^N_{i, t}}\left ( \vec s_{i, t}, \left \{ \vec \nu_{j, t - 1} \right \}_{j \in \cN_i} \right ) = x\right]} {\Pr \left [ \psi_{\cM^N_{i, t}}\left ( \vec s_{i, t}', \left \{ \vec \nu_{j, t - 1}' \right \}_{j \in \cN_i} \right ) = x\right]} \right ) \right | \le \eps & \text{ for all } \left ( \vec s_{i, t}, \left \{ \vec \nu_{j, t - 1} \right \}_{j \in \cN_i} \right ), \left ( \vec s_{i, t}', \left \{ \vec \nu_{j, t - 1}' \right \}_{j \in \cN_i} \right ) \in \cS \times \Rbb^{\deg(i)} \\ & \text{ s.t. } \left \| \left ( \vec s_{i, t}, \left \{ \vec \nu_{j, t - 1} \right \}_{j \in \cN_i} \right ) - \left ( \vec s_{i, t}', \left \{ \vec \nu_{j, t - 1}' \right \}_{j \in \cN_i} \right ) \right \|_1 \le 1
\end{align*}
}
for all $x \in \Rbb$.

\ppar{Central vs. Local Privacy.} In a local privacy scheme, the DP noise of the measurements can occur at the agent level by adding noise to the collected signals. Noise may be added to the signals after measurement to protect them against the revelations of belief exchange. The central scheme assumes a trusted environment where signal measurements can be collected without privacy concerns, but to the extent that protecting signals from the revelations of beliefs exchange is concerned, these methods would be equivalent.     

\section{Minimum Variance Unbiased Estimation}\label{sec:MVUE} 

\subsection{Minimum Variance Unbiased Estimation with Signal DP}\label{sec:min_var_unbiased_estimation_signal}

We present our first algorithm, which considers Minimum Variance Unbiased Estimation (MVUE). In this task, we aim to learn the MVUE of $\vec m_\theta$, i.e., to construct the estimate $\hat {\vec m_\theta} = (1 / n) \sum_{i = 1}^n \xi(\vec s_i)$ through local information exchange. The non-private version of this algorithm is presented in \cref{alg:non_private_min_var_unbiased_estimation_signal} according to which, the agents start with some private signals $\{ \vec s_i \}_{i \in [n]}$, calculate the sufficient statistics $\{ \xi(\vec s_i) \}_{i \in [n]}$ and then exchange these initial estimates with their local neighbors. \cref{alg:non_private_min_var_unbiased_estimation_signal} converges to $\hat {\vec m_\theta}$ in $t = O \left ( \frac {\log (n M_n / \rho)} {\log (1 / \beta^\star)} \right )$ steps to $\rho$-accuracy, which depends on the number of nodes $n$, the maximum absolute value $M_n$ of the sufficient statistics, and the spectral gap $\beta^\star$ of the transition matrix $A$. 

In its DP version, the algorithm proceeds similarly to the non-DP case, except each agent $i$ adds noise $\vec d_i$ to their sufficient statistic $\xi(\vec s_i)$. As we show later, to respect $\eps$-DP, the noise $\vec d_i$ depends on the agent's realized signal $\vec s_i$, the sufficient statistics $\xi(\cdot)$, and the privacy budget $\eps$. We provide the algorithm for the differentially private in \cref{alg:min_var_unbiased_estimation_signal}:

\begin{figure}[ht]
\noindent
 \begingroup\fboxsep=10pt
\fbox{
\begin{minipage}{0.95\textwidth}
\captionof{algorithm}{Minimum Variance Unbiased Estimation with Signal/Network DP} \label{alg:min_var_unbiased_estimation_signal}
The agents initialize with $\vec\nu_{i,0} = \xi(\vec s_i) + \vec d_i$ where $\vec d_i \sim \cD_i$ ($\cD_i$ is an appropriately chosen noise distribution), and in any future time period the agents communicate their values and update them according to the following rule:  
\begin{align}
{\vec\nu}_{i,t} =  a_{ii} \, {\vec\nu}_{i,t-1} + \sum_{j\in\mathcal{N}_i }a_{ij}{{\vec\nu}_{j,t-1}}. \label{eq:estimateUpdateAvgConcensus} 
\end{align} 
\end{minipage}}
\endgroup
\vspace{-10pt}
\end{figure}

Regarding convergence, in \cref{theorem:cop_min_var_unbiased_estimation_signal}, we prove that the convergence error is incurred due to two sources. The first source of error is the error due to the omniscient observer, which is the same as in the non-DP case, and the second source of error is incurred due to the privacy noise. Briefly, the latter term can be roughly decomposed to correspond to two terms: the former term is due to estimating the minimum variance unbiased estimator due to the noise, i.e. $(1 / n) \sum_{i = 1}^n \vec d_i$ and is vanishing with a rate proportional to $\log \left( \sum_{i = 1}^n \var {} {\vec d_i} \right ) / \log(1 / \beta^\star)$, and an additional non-vanishing term which is due to the mean squared error of $(1 / n) \sum_{i = 1}^n \vec d_i$ which corresponds to the sum of the variances of the signals. 

To minimize the convergence error, it suffices to minimize each variance $\var {} {\vec d_i}$ subject to $\eps$-DP constraints. By following recent results on the DP literature (cf. \cite{koufogiannis2015optimality}) we deduce that the variance minimizing distribution under $\eps$-DP constraints are the Laplace distributions with parameters $\Delta / \eps$, where $\Delta$ corresponds to the signal sensitivity (see below) and $\eps$ is the privacy budget. We present our Theorem (proved in \cref{proof:theorem:cop_min_var_unbiased_estimation_signal}):

\begin{theorem}[Minimum Variance Unbiased Estimation with Signal DP] \label{theorem:cop_min_var_unbiased_estimation_signal}
    The following hold for \cref{alg:min_var_unbiased_estimation_signal}:
    \begin{enumerate}
        \item For all $t$ and any zero-mean zero-mean distributions   
        {{\begin{align}
            \ev {} {\| \ovec \nu_t - \mathds 1 \hat {\vec m_\theta} \|_2} \le (1 +   \sqrt{(n - 1)} (\beta^\star)^{t} )  \sqrt {\sum_{j = 1}^n \var {} {\vec d_j}} + \sqrt{n(n - 1)} (\beta^\star)^{t} M_n,
            \label{eq:MVUE-signal-DP-bound}
        \end{align}}}
        where ${M}_n$ $ =$ $ \max\limits_{i\in[n]}|\xi(\vec{s}_i)|$ and $\beta^{\star} = \max\{\lambda_2(A),|\lambda_n(A)|\}$.
        \item  The optimal distributions $\{ \cD_i^\star \}_{i \in [n]}$ that minimize the MSE for each agent are the Laplace distribution with parameters $\Delta / \eps$ where $\Delta$ is the global sensitivity of $\xi$ Subsequently, ${\mathrm{TE}}( \cM^S) =  O \left ( n (\beta^\star)^t \left ( M_n + \frac {\Delta} {\eps} \right ) + \frac \Delta \eps \right )$. 
        % \item Under the optimal distributions, the estimates $\vec \nu_{i, t}$ are $\eps$-DP for all $t$
    \end{enumerate}

\end{theorem}

{
\begin{proofsketch}
    We do an eigendecomposition on $A$ and prove that since the noise is independent among agents, we can show that {{$\ev {} {\| \ovec \nu_t - \ovec \mu_t \|_2^2} = \sum_{j = 1}^n \var {} {\vec d_j} \sum_{i = 1}^n \lambda_i^{2t}(A) \vec q_{ij}^2$}}. By using Jensen's inequality, we can show that {{$\ev {} {\| \ovec \nu_t - \ovec \mu_t \|_2} \le \sqrt {\sum_{j = 1}^n \var {} {\vec d_j}} \left (1 + \sqrt {n - 1} (\beta^\star)^t \right )$}}. To minimize the convergence time, it suffices to minimize the variance of $\vec d_j \sim \cD_j$ subject to privacy constraints for all $j \in [n]$, resulting in an optimization problem which we resolve by applying the main result of \cite{koufogiannis2015optimality}. 
\end{proofsketch}
}

Finally, we note that similarly, in the multi-dimensional case, the sensitivity would be {{$\Delta = \max_{\vec s \in \cS} \left \| \nabla_{\vec s} \xi(\vec s) \right \|_\infty$}}. 

\ppar{A Note about Global Sensitivity of $\xi$.} In the above Theorem, it is tempting to think that the local sensitivity of each agent, i.e., {{$\Delta \xi_i = \left | {d \xi(\vec s_i)} /{d \vec s_i}\right |$}} can be used to calibrate the noise distribution that preserves differential privacy and has the minimum variance. However, as it has been shown in \cite{nissim2007smooth}, releasing noise that depends on the local sensitivity can compromise the signal. However, there are cases where global sensitivity is unsuitable for the learning task. For instance, in many distributions, such as the log-normal distribution, the global sensitivity may be unbounded (e.g., $\xi (\vec s) = \log \vec s$ for the log-normal and the sensitivity is $+\infty$ in this case). A possible solution for this situation and many exponential family distributions is to use another sensitivity instead of the global sensitivity.  \cite[Definition 2.2]{nissim2007smooth} proposes the use of the $\gamma$-smooth sensitivity. The way of constructing the $\gamma$-smooth sensitivity is to calculate $$S_{\xi, \gamma}^*(\vec s) = \max_{k > 0} \left \{ e^{-\gamma k} \max_{\vec s': \| \vec s' - \vec s \|_1 = k} |\xi(\vec s') - \xi(\vec s) | \right \}.$$ Using the Laplace mechanism with parameters ${2 S_{\xi, \gamma}^\star(\vec s)} /\eps$ for $\gamma = {\eps}/ {(2 \log (2 / \delta))}$ would guarantee $(\eps, \delta)$-DP; see \cite[Corollary 2.4]{nissim2007smooth}. 

For example, we will consider the case of mean estimation in a log-normal distribution with known variance, a common task in many sensor networks, as we argue in \cref{sec:experiments}. The sufficient statistic in this case is $\xi(\vec s) = \log \vec s$, and the local sensitivity at distance $k$ can be computed to be $k / \vec s$. Therefore, the smooth sensitivity is (note the global sensitivity in this case is unbounded):
{{
\begin{align} \label{eq:smooth_sensitivity_log_normal}
    S_{\xi, \gamma}^*(\vec s) = \max_{k > 0} \left \{  e^{-\beta k} \frac {k} {\vec s} \right \} = \frac 1 {e \gamma \vec s} = \frac {2 \log (2 / \delta)} {e \eps \vec s}, \quad {\text{for }} \vec s > 0.
\end{align}}} If agents have access to several signals, and the exact formula of $\xi$ is not known, the sensitivity can still be approximated via samples as shown in \cite{wood1996estimation}.
 
% \vspace{-15pt}
\ppar{Remark for Network DP.} Note that to achieve Network DP, one natural algorithm is to add noise $\vec d_{i, t}$ at each round of \cref{alg:min_var_unbiased_estimation_signal} at a high enough level to protect both the network neighborhoods and the private signals. The issue with this algorithm is that it is divergent, i.e., {$\ev {} {\left \| \ovec \nu_t \right \|_2^2} \to \infty$}, because the estimate at time $t$ is {$\ovec \nu_t = A^t \ovec \xi + \sum_{\tau = 0}^{t - 1} A^\tau \ovec d_{t - \tau}$}, and its mean squared error, i.e., {$\ev {} {\left \| \sum_{\tau = 0}^{t - 1} A^\tau \ovec d_{t - \tau} \right \|_2^2}$}, grows linearly with $n$ and $t$. To avoid the accumulation of the DP noise, we should limit the DP noising to the initial step, which will achieve a bounded error because the mixing matrix, $A$, is doubly stochastic. To this end, we choose the noise level to satisfy $\eps$-DP with the aim of protecting both signals and network connections --- and run \cref{alg:min_var_unbiased_estimation_signal} at the new noise level. The error of this algorithm will be identical to the error bound of \cref{theorem:cop_min_var_unbiased_estimation_signal}. However, the sensitivity of the noise should be set to accommodate both network and signal dependencies as follows: $\Delta \nu_{i}^{\cM^N} = \max \left \{ \Delta, \max_{j \in \cN_i} a_{ij} \right \}$, and the optimal distributions will be $\cD_i^\star = \mathrm{Lap} \left ({\max \left \{ \Delta, \max_{j \in \cN_i} a_{ij} \right \}} /{\eps} \right )$; see \cref{corollary:cop_min_var_unbiased_estimation_network}. In the case that $\Delta$ is unbounded, we can replace $\Delta$ with the smooth sensitivity accounting for the network effects, i.e., $\max \left \{ \max_{j \neq i} a_{ij}, S_{\xi, \gamma}^*(\vec s_i) \right \}$ for each signal $\vec s_i$ and get an $(\eps,\delta)$-DP algorithm. Note that private signals will remain DP-protected by the post-processing immunity of the Laplace mechanism. In \cref{proof:corollary:cop_min_var_unbiased_estimation_network}, we use induction to show that Network DP is preserved at the $\eps$ level for all times $t$ when the mixing matrix $A$ is non-singular. The following corollary summarizes our results: 

\begin{corollary}[Total Error of Minimum Variance Unbiased Estimation with Network DP] \label{corollary:cop_min_var_unbiased_estimation_network}
Assume the mixing matrix $A$ is non-singular, let $\Delta$ be the global sensitivity of $\xi$, and $a = \max_{i \neq j} a_{ij}$. The Total Error of \cref{alg:min_var_unbiased_estimation_signal} with Network DP satisfies {{$${\mathrm{TE}}( \cM^N) =  O \left ( n (\beta^\star)^t \left ( M_n + \frac {\max \{ a, \Delta \}} {\eps} \right ) + \frac {\max \{a, \Delta\}} \eps \right ).$$}}
\end{corollary} 

\vspace{-5pt}

\ppar{Tightness of Analysis.} We note that the above analysis is tight because there exists an instance for which the upper bound is precisely achieved (for any noise distribution). To show this, consider the complete graph $K_n$, which corresponds to weights $a_{ij} = 1 / n$ for all $i, j \in [n]$ with a spectral gap of zero, and $\vec s_i \sim \cN(0, 1)$. The network dynamics on the complete graph converge in one iteration, and it is straightforward to show that the total error (TE) converges to $1/\eps$ as $n\to\infty$ for any noise distributions subject to $\eps$-DP constraints (by the central limit theorem), i.e., \cref{supp:eq:mvue_total_bound} in the Supplementary Materials holds with equality. 

\section{Online Learning of Expected Values} \label{sec:online_learning_dp}

\subsection{Privacy Frameworks}

We consider the online learning framework where the agents aim to learn the common expected value $\vec m_{\theta} = \ev {\theta} {\xi(\vec s)}$ of the sufficient statistics of their signal distributions. In this regime, the agents observe signals $\vec s_{i, t}$ at every time $t \ge 1$ and update their estimates $\vec \nu_{i, t}$ by weighing the information content of their most recent private signals, $\xi(\vec s_{i, t})$, their previous estimates, $\vec \nu_{i, t - 1}$, and the estimates of their neighbors, $\{ \vec \nu_{j, t - 1} \}_{j \in \cN_i}$. The mechanisms that we analyze in this section will accommodate two types of privacy needs with convergence and $\eps$-DP guarantees for the agents: 
{
\begin{enumerate}
    \item[$\cM^S$:] \emph{Signal DP (\cref{alg:online_learning_dp_signal})}. Here the agent adds noise to privatize their belief $\vec \nu_{i, t}$ with respect to their signal. To assert the consistency of the estimator, the agent averages her previous estimate and the estimates of her neighbors with weight $(t - 1) / t$ and her signal with weight $1/t$, similarly to \cref{alg:non_private_online_learning_dp}. Therefore, the local sensitivity of this mechanism is the sensitivity of the sufficient statistic weighted by $1/t$ and equals $\Delta \nu_{i, t}^{\cM^S} = \frac {\Delta} t$.

    \item [$\cM^N$:] \emph{Network DP} (\cref{alg:online_learning_dp_network}). Here we consider the protection of the agent's signals $\vec s_{i, t}$ together with their local neighborhood, namely the neighboring beliefs $\left \{ \vec \nu_{j, t - 1} \right \}_{j \in \cN_i}$. We note that when deciding the noise level at time $t$, we do not need to include the agent's own belief from time $t - 1$ in the DP protection; $\nu_{i,t-1}$ including all the private signals up to and including at time $t-1$ remain protected by the property of the post-processing immunity. If we were to use $\cref{alg:online_learning_dp_signal}$ for Network DP the sensitivity of the mechanism would be $\max \left  \{ \frac {\Delta} {t}, \frac {t - 1} {t} \max_{j \in \cN_i} a_{ij} \right \}$ which approaches 1 as $t \to \infty$ and violates the consistency of the estimator $\vec \nu_{i, t}$. For this reason, we need to adapt the weighting scheme of \cref{alg:online_learning_dp_signal} to be consistent and respect Network DP. We give more details of the altered algorithm (\cref{alg:online_learning_dp_network}) in \cref{sec:online_learning_dp_network}.
\end{enumerate}
}

\begin{figure}[ht]
\noindent
 \begingroup\fboxsep=10pt
\fbox{
\begin{minipage}{0.95\textwidth}
\captionof{algorithm}{Online Learning of Expected Values with Signal DP} \label{alg:online_learning_dp_signal}
In any time period $t\geq 1$ the agents observe a signal $\vec{s}_{i,t}$, and update their estimates according to the following rule:{  
\begin{align}
{\vec\nu}_{i,t} = & \frac{t-1}{t}\left(a_{ii} \, {\vec\nu}_{i,t-1} + \sum_{j\in\mathcal{N}_i }a_{ij}{{\vec\nu}_{j,t-1}}\right) + \frac{1}{t} (\xi(\vec{s}_{i,t}) + \vec d_{i, t}), \label{eq:estimateUpdateAvgConcensusOnline-signal-dp} 
\end{align}} where {$\vec d_{i, t} \sim \cD_{i, t}$} is appropriately chosen noise.

\end{minipage}}
\endgroup 
\end{figure}

\subsection{Online Learning of Expected Values with Signal DP} \label{sec:online_learning_dp_signal}

Here we will analyze the performance of \cref{alg:online_learning_dp_signal} where agents only protect their signals and present the corresponding error analysis. The error of the estimates $\vec \nu_{i, t}$ compared to the expected value $\vec m_\theta = \ev {\theta} {\xi(\vec s)}$, again consists of two terms; one due to decentralization and the statistics of the signals themselves (CoD), and another due to the variances of the added DP noise variables (CoP). Each of these two terms decays at a rate of $1 / \sqrt t$. They, in turn, can be bounded by the sum of two terms: a constant that is due to the principal eigenvalue of $A$ and represents the convergence of the sample average to $\vec m_{\theta}$, and $n - 1$ terms due to $|\lambda_i(A)|$ for $2 \le i \le n$ which dictate the convergence of the estimates $\vec \nu_{i, t}$ to their sample average. The latter depends on the number of nodes $n$ and the spectral gap $\beta^\star$ of matrix $A$. We formalize our results as follows (proved in \cref{proof:theorem:cop_online_learning_dp_signal}): 

\begin{theorem}[Online Learning of Expected Values with Signal DP] \label{theorem:cop_online_learning_dp_signal}
    The following hold for \cref{alg:online_learning_dp_signal} and mechanism $\cM^S$:
    \begin{enumerate}  
        \item For every time $t$, and all distributions $\{ \cD_{i, t} \}_{i \in [n], t \ge 1}$ we have that 
        {
        \begin{align*}
            \ev {} {\left \| \ovec \nu_t - \mathds 1 \vec m_{\theta} \right \|_2} \le \frac 1 t \left ( \sqrt{nt \var {} {\xi(\vec s)}} + \sqrt {\sum_{j = 1}^n \sum_{\tau = 0}^{t - 1} \var {} {\vec d_{j, t - \tau}}} \right )\left ( 1 + \sqrt {\frac {n - 1} {1 - (\beta^\star)^2}} \right ) .
        \end{align*}}
        \item The optimal distributions {$\{ \cD_{i, t}^\star \}_{i \in [n], t \ge 1}$} are  {$\cD_{i, t}^\star = 
        \mathrm {Lap} \left ( {\Delta} /{\eps} \right )$}, where $\Delta$ is the global sensitivity. Moreover, we have: {${\mathrm{TE}}(  \cM^S)  = O \left ( \frac {n} {\sqrt t} \left ( \sqrt {\var {} {\xi(\vec s)}} + \frac {\Delta} {\eps} \right )\right )$}.
    \end{enumerate}
\end{theorem}
{
\begin{proofsketch}
    We note that $\ovec \nu_t = \frac 1 t \sum_{\tau = 0}^{t - 1} A^\tau (\ovec \xi_{t - \tau} + \ovec d_{t - \tau})$. We decompose $A$ as $A = Q \Lambda Q^T$ where $\Lambda$ is the eigenvalue matrix and $Q$ is the orthonormal eigenvector matrix, and upper bound $\ev {} {\left \| \ovec \nu_t - \ovec \mu_t \right \|_2^2}$ as a weighted sum of powers of the eigenvalues of $A $ and the sum of the variances of $\ovec d_{t - \tau}$, i.e., $ \ev {} {\left \| \ovec \nu_t - \ovec \mu_t \right \|_2^2}  \le \frac 1 {t^2} \sum_{i = 1}^n \sum_{\tau = 0}^{t - 1} \lambda^{2 \tau}_i(A) \ev {} {\left \| \ovec d_{t - \tau} \right \|_2^2}$. We apply the Cauchy-Schwarz inequality again and bound this term with the sum of all variances across all agents $j \in [n]$ and all the rounds $0 \le \tau \le t - 1$, plus the sum of the powers of the eigenvalues of $A$, i.e., $\sum_{i = 1}^n \sum_{\tau = 0}^{t - 1} \lambda_i^{2\tau}(A)$. The latter term is decomposed into a term that depends on the principal eigenvalue $\lambda_1(A) = 1$ and $n - 1$ terms that are dominated by the spectral gap $\beta^\star$. We do the same analysis for $\ovec \mu_t - \mathds 1 \vec m_{\theta}$, and we finally apply Jensen's inequality and the triangle inequality to get the final bound. To optimize the bound, it suffices to minimize the variances $\var {} {\vec d_{i, t - \tau}}$ for all $0 \le \tau \le t - 1$ and $i \in [n]$ subject to $\eps$-DP constraints. The optimal noise distributions are derived by solving the same optimization problems as in \cref{theorem:cop_min_var_unbiased_estimation_signal} for all $i \in [n]$ and all $t \ge 1$. Calculating the bound for Laplace noise with parameters $\Delta / \eps$ gives the upper bound on ${\mathrm{TE}}(  \cM^S)$.   
\end{proofsketch}
}

\subsection{Online Learning of Expected Values with Network DP} \label{sec:online_learning_dp_network}

Above, we briefly discussed why a learning algorithm that puts weights $\frac {t - 1} {t}$ on the network predictions and $\frac 1 t$ on the private signal would not work (and in fact, the dynamics become divergent in that case). In \cref{alg:online_learning_dp_network}, we present a different learning scheme that uses weights $\frac 1 t$ for the private signals and neighboring observations and $1 - \frac 1 t (2 - a_{ii})$ for the previous beliefs of the agent. The motivation behind this learning scheme is that the sensitivity is now going to be $\Delta \vec \nu_{i, t}^{\cM^N} = \frac 1 t \max \left \{ \max_{j \in \cN_i} a_{ij}, \Delta \right \}$ which goes to zero, instead of $\max \left \{ \frac {t - 1} {t} \max_{j \in \cN_i} a_{ij}, \frac 1 t  \Delta \right \}$ which approaches 1. The drawback is that the added self-weight on one's own beliefs at every time step will slow down the mixing time and convergence. In the sequel, we present \cref{alg:online_learning_dp_network} and analyze its performance.

\begin{figure}[ht]
\noindent
 \begingroup\fboxsep=10pt
\fbox{
\begin{minipage}{0.95\textwidth}
\captionof{algorithm}{Online Learning of Expected Values with Network DP} \label{alg:online_learning_dp_network}
In any time period $t\geq 1$ the agents observe a signal $\vec{s}_{i,t}$ and update their estimates according to the following rule:{   
\begin{align}
{\vec\nu}_{i,t} = &  \left ( 1 - \frac 1 t (2 - a_{ii}) \right ) \, {\vec\nu}_{i,t-1}  + \frac 1 t \left (\sum_{j\in\mathcal{N}_i }a_{ij}{{\vec\nu}_{j,t-1}} 
 + \xi(\vec{s}_{i,t}) \right )+ \frac 1 t \vec d_{i, t},\label{eq:estimateUpdateAvgConcensusOnline-network-dp} 
\end{align}} where {$\vec d_{i, t} \sim \cD_{i, t}$} is appropriately chosen noise.
\end{minipage}}
\endgroup 
\end{figure}
\vspace{-10pt}
We can write the above system in matrix notation as $\ovec \nu_{t} = \left ( B(t) - \frac 1 t I \right ) \ovec \nu_{t - 1} + \frac 1 t \ovec \xi_t + \frac 1 t \ovec d_t$ where $b_{ii}(t) = 1 - \frac 1 t (1 - a_{ii})$ and $b_{ij}(t)  = \frac 1 t a_{ij}$ for all $j \neq i$. First, we study the convergence of \cref{alg:online_learning_dp_network} when no noise is added, i.e., $\ovec d_t = 0$. Note that similarly to \cref{alg:online_learning_dp_signal}, the error term is comprised of two terms, one owing to the principal eigenvalues of $C(t) = B(t) - \frac 1 t I$, i.e., $\lambda_1(C(t)) = 1 - 1 / t$ which controls the convergence of the sample average of the estimates to $\vec m_{\theta}$, and $n - 1$ terms due to $|\lambda_i(C(t))|$ which control the convergence of the estimates $\vec \nu_{i, t}$ to their sample average.  

% Note that $B(t)$ is doubly stochastic. 

\begin{theorem}[Online Learning of Expected Values with Network DP] \label{theorem:cop_online_learning_dp_network}
     For \cref{alg:online_learning_dp_network}, the following hold:

    \begin{enumerate}
        \item For all $t \ge 1$ and all distributions $\{ \cD_{i, t} \}_{i \in [n], t \ge 1}$, we have that {{$$\ev {} {\left \| \ovec \nu_t - \mathds 1 \vec m_\theta \right \|_2} \le \frac {1} {t} \left ( \sqrt {nt \var {} {\xi(\vec s)}} + \sqrt { \sum_{i = 1}^n \sum_{\tau = 0}^{t - 1} \var {} {\vec d_{i, t - \tau}}} \right ) \left ( 1 + \sqrt {\frac {n - 1} {3 - 2 \beta^\star}}  \right ).$$}}

        \item The optimal distributions  $\{ \cD_{i, t}^\star \}_{i \in [n], t \ge 1}$ that minimize the MSE bound subject to $\eps$-DP are the Laplace Distributions with parameters ${\max \left \{ \max_{j \in \cN_i} a_{ij}, \Delta \right \}}/ {\eps}$. Moreover, if $\Delta$ is the global sensitivity and $a = \max_{i \neq j} a_{ij}$, then {{$${\mathrm{TE}}( \cM^N) = O \left ( \frac {n} {\sqrt t} \left ( \frac {\max \{a, \Delta \}} {\eps} + \sqrt {\var {} {\xi(\vec s)}} \right )\right ).$$}}
    \end{enumerate}

\end{theorem}

\begin{proofsketch}
    Our proof (\cref{proof:theorem:cop_online_learning_dp_network}) follows a similar analysis to \cref{theorem:cop_online_learning_dp_signal}. We first note that $C(t)$ can be written as {{$C(t) = \frac {t - 2} {t} I + \frac 1 t A$}}, and therefore shares the same eigenvectors with $A$ and has eigenvalues {{$\lambda_i(C(t)) = \frac {t - 2} t + \frac 1 t \lambda_i (A)$}}. We define {{$\Phi(t) = \prod_{\tau = 0}^{t - 1} C(\tau)$}}, and show that $\lambda_i(\Phi(t)) \le t^{\lambda_i(A) - 2}$. Moreover, {{$\ovec \nu_t = \frac 1 t \sum_{\tau = 0}^{t - 1} \Phi(\tau) (\ovec \xi_{t - \tau} + \ovec d_{t - \tau})$}}. By using the bound on the eigenvalues of $\Phi(\tau)$ for $0 \le \tau \le t - 1$ and by applying the same analysis as in \cref{theorem:cop_online_learning_dp_signal}, we decompose {{$\ev {} {\left \| \ovec \nu_t - \ovec \mu_t \right \|_2^2}$}} and get that the sum of the powers of the eigenvalues consists of a term due to the principal eigenvalues, which decays with the rate $1 / t^2$ and $n - 1$ terms that decay as $\frac {n - 1} {t^2 (3 - 2 \beta^\star)}$. We finally deduce that {{$ \ev {} {\left \| \ovec \nu_t - \ovec \mu_t \right \|_2} \le \frac 1 t \sqrt{ \sum_{i = 1}^n \sum_{\tau = 0}^{t - 1} \var {} {\vec d_{i, t - \tau}}}   \left ( 1 + \sqrt {\frac {n - 1} {3 - 2 \beta^\star}}  \right )$}}. The same bound holds for {{$\ovec q_t = \ovec \mu_t - \mathds 1 \vec m_{\theta}$}}, and by applying Jensen's inequality and the triangle inequality, we get the error bound. The optimization of the noise variables follows similar logic to \cref{theorem:cop_online_learning_dp_signal}, with a different sensitivity. 
\end{proofsketch}

\ppar{Tightness of analysis.} We note that the analysis is tight, and the tight example is precisely the same as the example we provided for MVUE. 

\begin{figure}[t]
    \centering
        \includegraphics[width=0.4\textwidth]{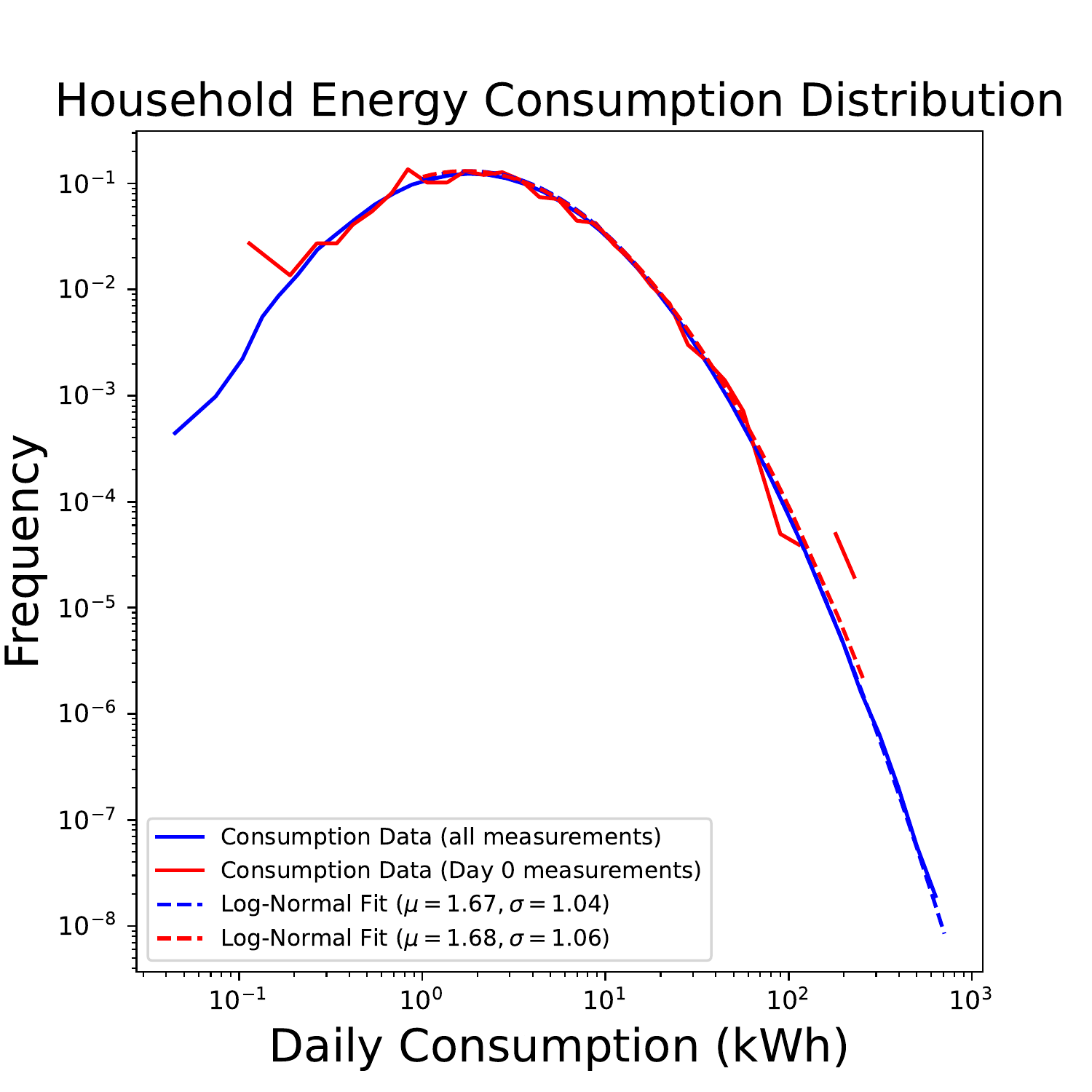}
    \includegraphics[width=0.4\textwidth]{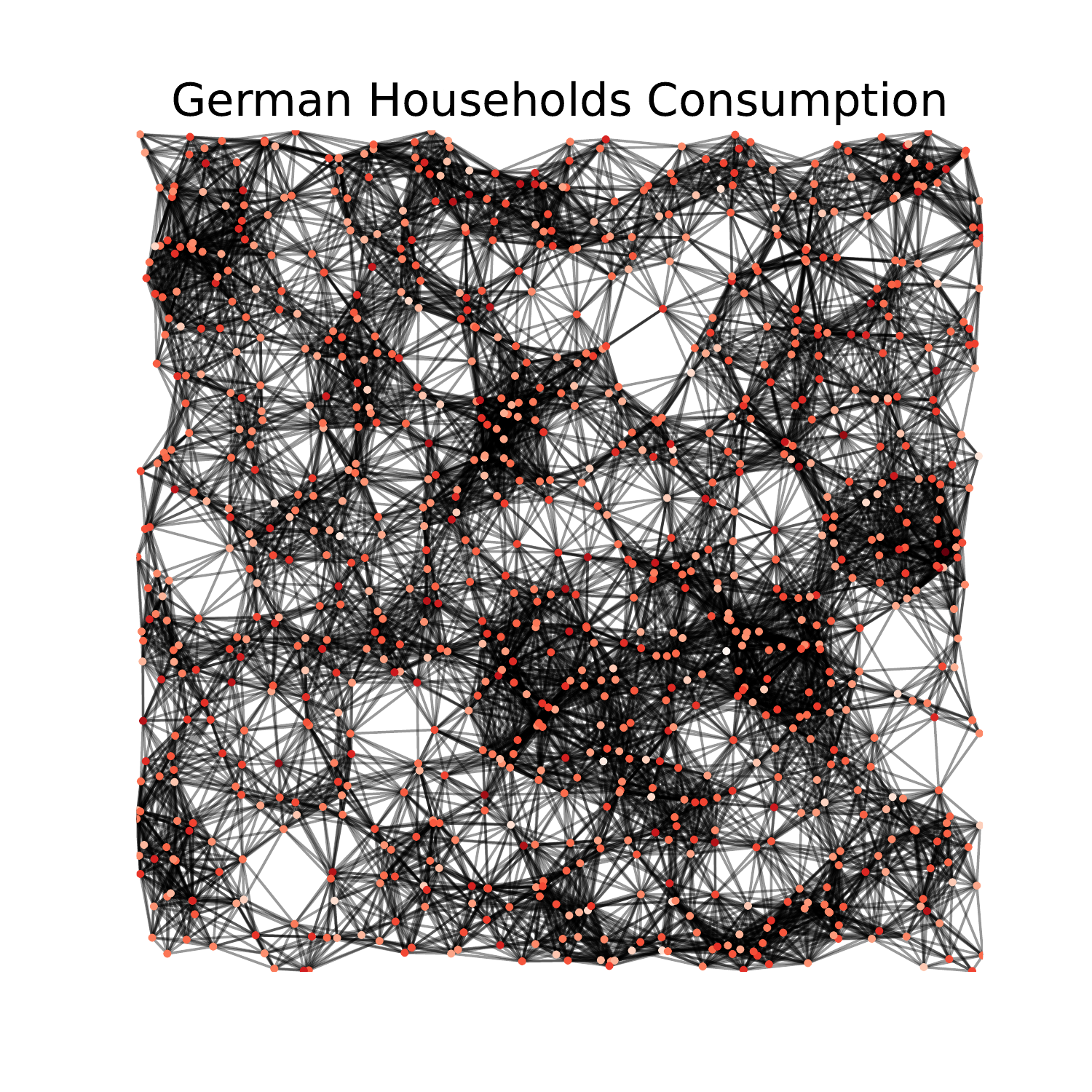}
    \includegraphics[width=0.4\textwidth]{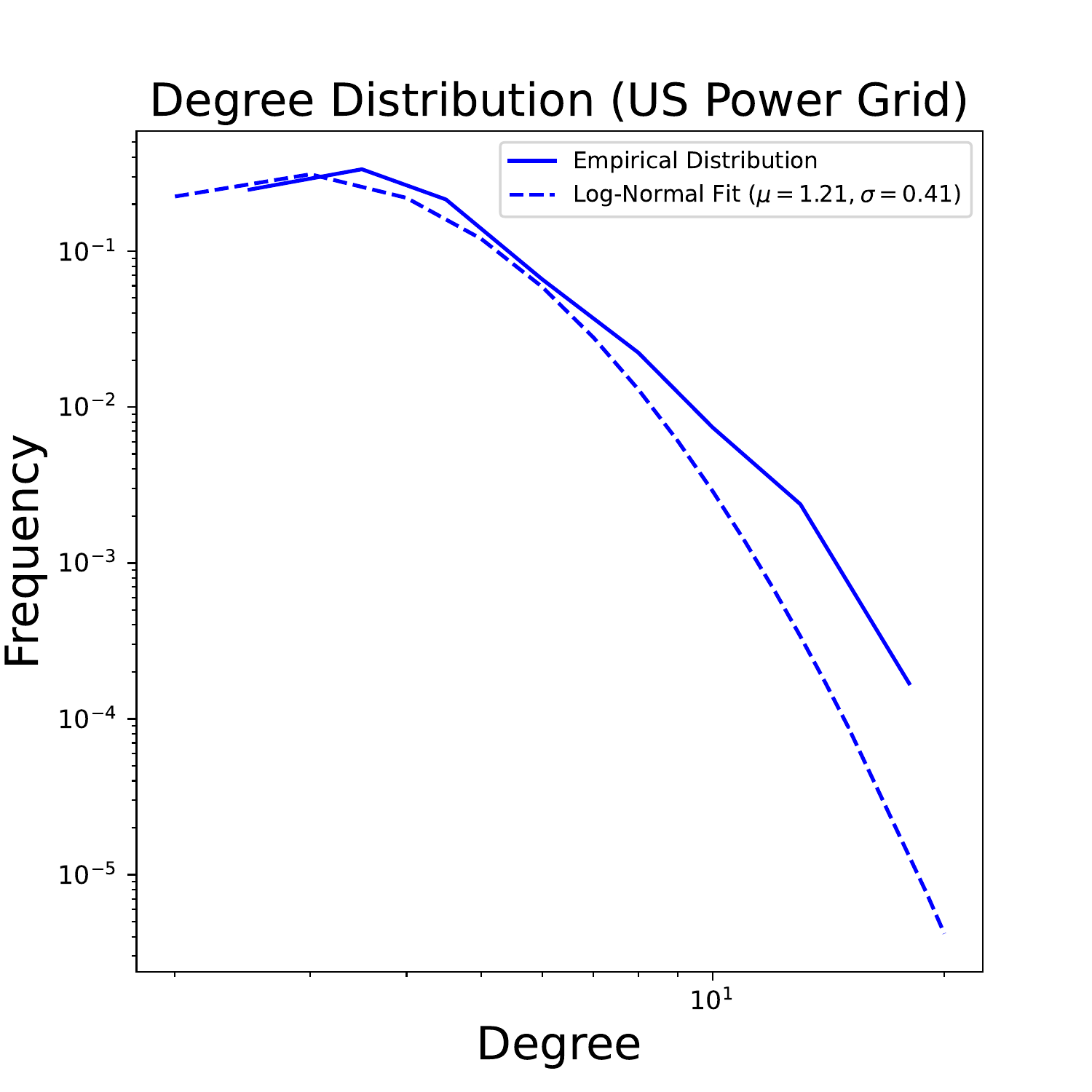}
    \includegraphics[width=0.4\textwidth]{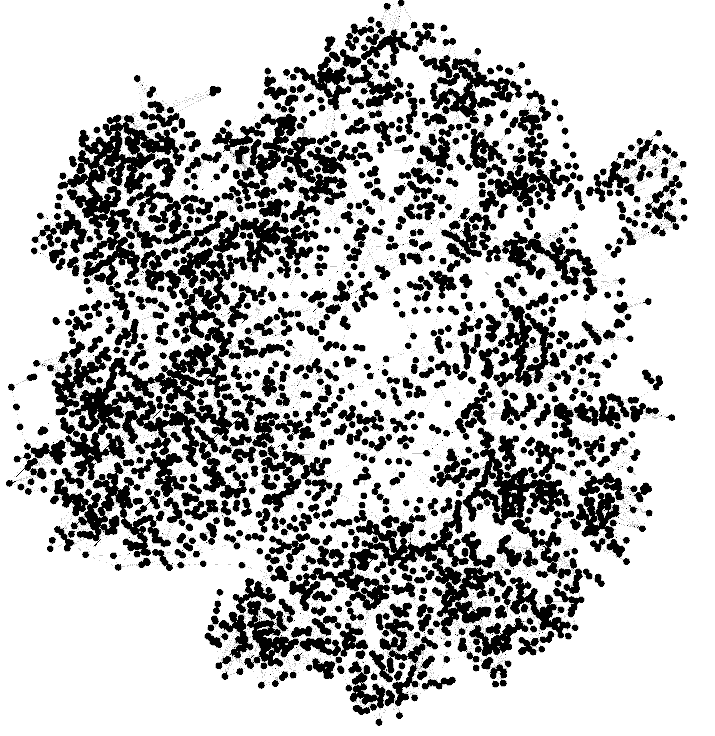}
    \caption{Distribution of Daily Consumption (in kWh) for the GEM House openData with log-normal fits is shown on the left (for all measurements and Day 0 measurements), followed by a visualization of the generated random geometric network with $\rho = 0.1$. The next two figures show the US Power Network degree distribution with log-normal fit, followed by its visualization.}
    \label{fig:distribution_power_networks}
\end{figure}

\section{Real-World Experiments} \label{sec:experiments}

\ppar{Datasets and Real-World Scenarios.} To showcase the effectiveness of our algorithms, i.e., convergence to the actual estimates subject to $(\eps, \delta)$-DP, we conduct two experiments that correspond to the estimation of power consumption in the electric grid.

The first case considers estimating power consumption via electricity measurements of individual households. Consumption behavior is considered highly sensitive. It can reveal compromising information about daily habits and family illnesses or pose a security threat if exploited by an adversary, e.g., to coordinate attack time. Here, we assume that each household faces a privacy risk in sharing their measurements, and they may decide to mitigate the privacy risks by adding noise to their estimates. The ability to estimate average consumption in a distributed manner is useful for distributed load balancing and deciding generation plans. 

For this scenario, we consider the GEM House openData dataset \cite{milojkovic2018gem}, which contains power consumption measurements of $n = 969$ individual households over $T = 1096$ days (i.e., three years). The dataset contains cumulative power consumption measurements $\vec c_{i, t}$ for each particular day. To extract the actual measurements (in kWh) $\vec s_{i, t}$ we take the differences between consecutive days and divide by $10^{10}$ (see Section 3.2.1 of \cite{milojkovic2018gem}), i.e., $\vec s_{i, t} = \frac {\vec c_{i, t} - \vec c_{i, t - 1}} {10^{10}}$. We observe that $\vec s_{i, t}$ follows a log-normal distribution with mean $\mu = 1.67$ and standard deviation $\sigma = 1.04$, as shown in \cref{fig:distribution_power_networks}. Moreover, in this dataset, the network structure is absent. For this reason, we generate a random geometric graph, i.e., we generate $n = 969$ nodes randomly distributed in $[0, 1]^2$ and connect nodes with a distance at most $\rho = 0.1$. Random geometric graphs have been used to model sensor networks \cite{kenniche2010random} and correspond to a straightforward criterion for determining links since they connect nearby households. For a fixed random seed (seed = 0), the network contains $m = 13,236$ edges and is visualized in \cref{fig:distribution_power_networks}.

The second dataset examines estimating power consumption in the US Power Grid Network from \cite{watts1998collective}. In this case, we hypothesize that each power station faces a privacy risk -- for example, vulnerability to a cyber attack -- in sharing their measurements and decides to reduce its privacy risk by adding noise. The network contains $n = 4,941$ nodes and $m = 6,594$ edges. \cref{fig:distribution_power_networks} shows the power network and its degree distribution. Here we artificially generate i.i.d. signals for $T = 100$ as $\vec s_{i, t} \sim \mathrm{LogNormal}(\mu = 10, \sigma=1)$.

\begin{figure}
    \centering  \includegraphics[width=0.75\textwidth]{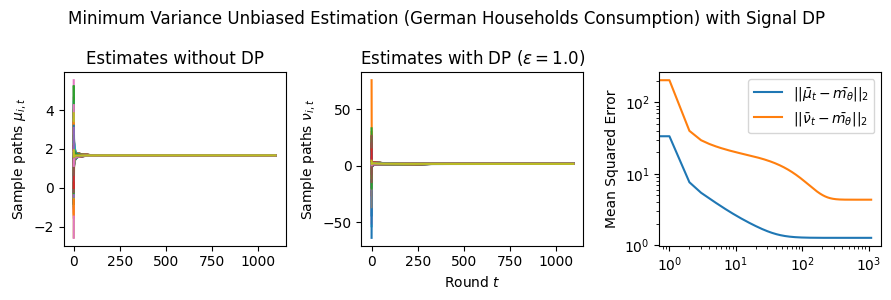} \includegraphics[width=0.75\textwidth]{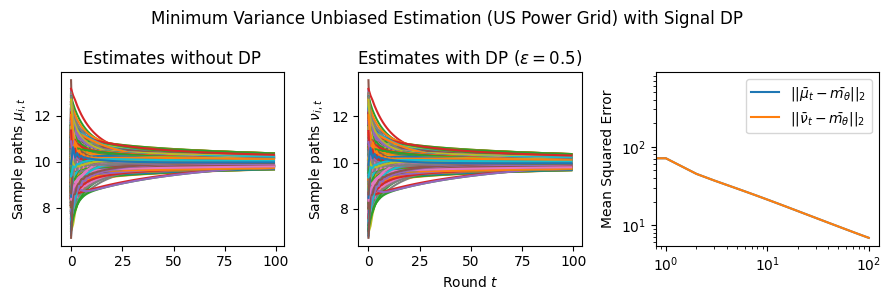}
    \caption{Sample Paths for MVUE with Signal DP. Note the large error in the case of the German household dataset is because protecting households with low (near zero) consumption rates even at a relatively high privacy budget ($\eps = 10$) comes at a huge cost to accuracy.}
    \label{fig:sample_paths_min_var_unbiased_estimation_signal}
\end{figure}

\begin{figure}
    \centering  \includegraphics[width=0.75\textwidth]{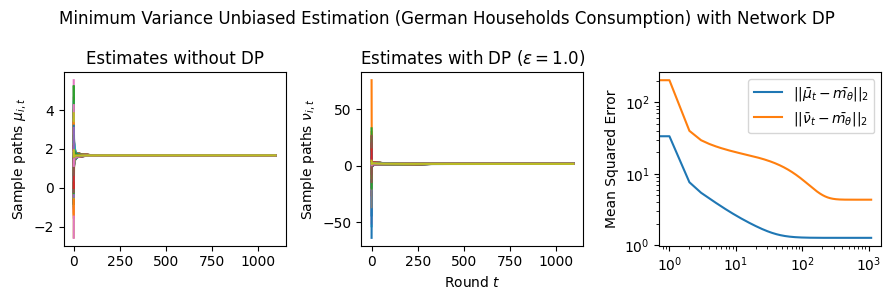} \includegraphics[width=0.75\textwidth]{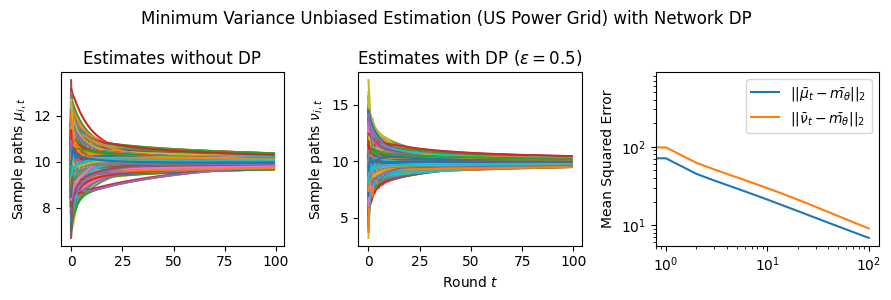}
    \caption{Sample Paths for MVUE with Network DP. Note that even requiring a moderate accuracy in the case of the German household dataset comes at a high cost to privacy ($\eps = 1$), pointing to the challenges of maintaining privacy when sensitivities cannot be locally bounded (some household consumption values are close to zero).}
\label{fig:sample_paths_min_var_unbiased_estimation_network}
\end{figure}

\begin{figure}
    \centering  \includegraphics[width=0.75\textwidth]{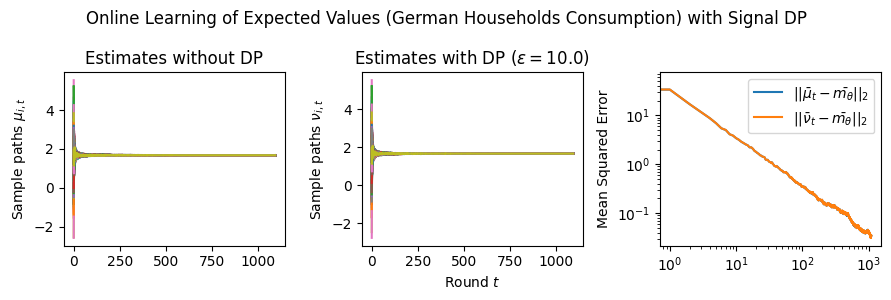} \includegraphics[width=0.75\textwidth]{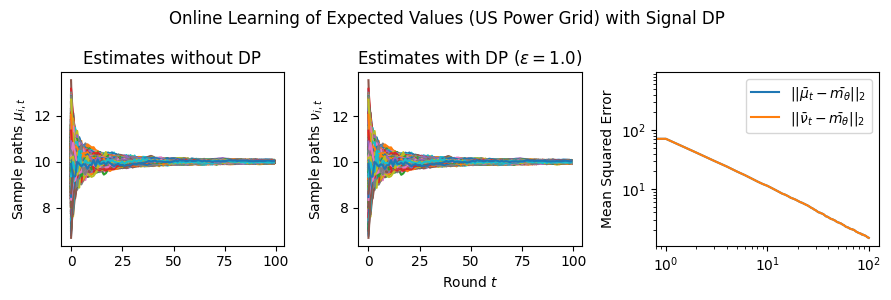}
    \vspace{-10pt}
    \caption{Sample Paths for Online Learning of Expected Values with Signal DP. Choosing $\eps$ large enough leads to a convergent behavior for the German household dataset, but no meaningful privacy protection can be afforded in that case ($\eps = 10$).}
    \label{fig:sample_paths_online_learning_dp_signal}
\end{figure}

\ppar{Estimation Tasks.} In both cases, the task is to estimate each log-normal distribution's mean $\mu$ in two scenarios. In the first scenario, we estimate the mean only from the initial measurements, i.e., estimate $\hat \mu_{\mathrm{MVUE}} = \frac {\sum_{i = 1}^n \log \vec s_{i, 1}} {n}$. \cref{fig:sample_paths_min_var_unbiased_estimation_signal} presents some sample paths for this task as the horizon $t$ varies, and \cref{fig:mse_plot} presents the final MSE after $T$ for various values of the privacy budget $\eps$. In the second scenario, we estimate the mean with online learning (OL) to estimate $\hat \mu_{\mathrm{OL}} = \frac {\sum_{i = 1}^n \sum_{t = 1}^T \log \vec s_{i, t}} {nT} \to \ev {} {\log \vec s} = \mu_{\mathrm{OL}}$. We run simulations in both regimes where we want to protect the signal and the network connections. Because in this case the global sensitivity is unbounded, we use the smooth signal sensitivities $S_{\xi, \gamma}^*(\vec s_{i, t}) = \frac {2 \log (2 / \delta)} {e \eps \vec s_{i, t}}$ for each signal $\vec s_{i, t}$ with $\delta = 0.01$. The resulting algorithms are $(\eps, \delta)$-DP (see \cref{sec:min_var_unbiased_estimation_signal}). 

\ppar{Comparison with \cite{rizk2023enforcing}.}  Finally, in \cref{app:rizk} we explain the adaptation of \cite{rizk2023enforcing} first-order DP consensus algorithm to both MVUE and OL tasks. \cref{fig:mse_plot} shows the comparison of MSE performances between our algorithms and \cite{rizk2023enforcing} given the same privacy budget per agent. Compared to \cite{rizk2023enforcing}, our method is able to achieve significantly smaller MSE under the same total privacy budget; $\approx 1000\times$ smaller for both datasets. While \cite{rizk2023enforcing}'s algorithm are applicable to a broader set of tasks than the MVUE and OL estimation setups presented here, their inclusion of private signals at every iteration entails DP noising at every step of the iteration and comes at a higher cost to accuracy. 

\begin{figure}[t]
    \centering  
    \includegraphics[width=0.75\textwidth]{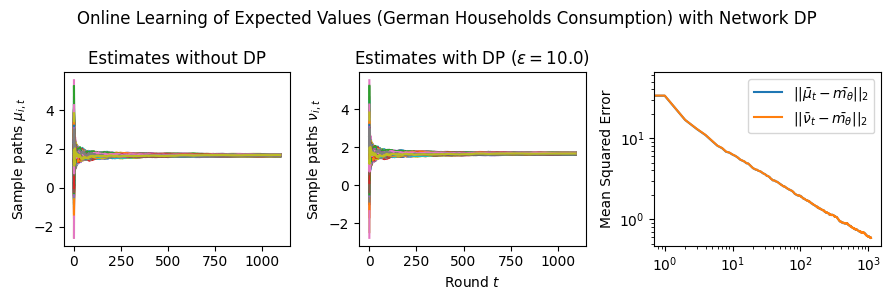} \includegraphics[width=0.75\textwidth]{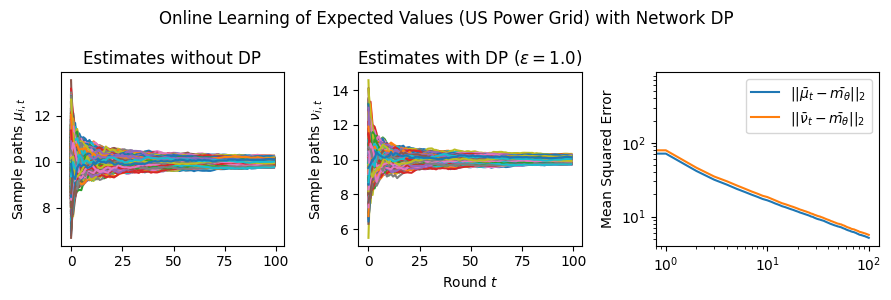}
    \caption{Sample Paths for the Online Learning of Expected Values with Network DP. Protecting network neighborhoods is a harder task than protecting private signals. While almost perfect signal DP can be achieved with reasonable accuracy for the US power grid network ($\eps = 1$ in \cref{fig:sample_paths_online_learning_dp_signal}), even moderate protection of network neighborhoods ($\eps = 1$) come at a noticeable cost to accuracy. Privacy protection for network neighborhoods in the case of German households is further complicated by the existence of almost zero signals with locally unbounded sensitivity and no meaningful protections is accomplished ($\eps = 10$). Privacy and accuracy, in this case,  become conflicting criteria that cannot be reconciled.}
    \label{fig:sample_paths_online_learning_dp_network}
\end{figure}

\section{Discussion and Conclusion}\label{sec:conc}

\ppar{Results and Insights.} In all cases of signal DP with the US power grid network, the DP noise did not affect the convergence rate in practice for this choice of signals, privacy budget $\eps$, and information leakage probability $\delta$. Also, we observe that \cref{alg:online_learning_dp_signal} converges faster than \cref{alg:online_learning_dp_network} (even in the absence of DP noise) because of the underlying mixing matrices, which are $\frac {t - 1} {t} A$, and $C(t) = \frac {t - 2} {t} I + \frac 1 t A$ respectively. Moreover, both of these algorithms converge faster (with and without DP noise) than \cref{alg:min_var_unbiased_estimation_signal}. This is expected since \cref{alg:min_var_unbiased_estimation_signal} has access to $n$ samples in total, while \cref{alg:online_learning_dp_signal,alg:online_learning_dp_network} have access to $nt$ signals and can bring the estimation error down by a $1 / t$ factor. Comparison of Signal DP (\cref{fig:sample_paths_min_var_unbiased_estimation_signal,fig:sample_paths_online_learning_dp_signal}) with Network DP (\cref{fig:sample_paths_min_var_unbiased_estimation_network,fig:sample_paths_online_learning_dp_network}) for MVUE and online learning tasks points to the increased difficulty of ensuring network DP: network privacy protections are harder to achieve and they imply signal protections automatically. On the other hand,  when the local sensitivities can grow large  --- as with the German household dataset --- maintaining privacy for households with low consumption comes at a huge cost to accuracy (see, e.g., \cref{fig:sample_paths_online_learning_dp_signal}). This is because for the log-normal distribution, $d \xi( \vec s)/d \vec s$ grows unbounded as $\vec s \to 0$.

\ppar{Extensions.} We extend our algorithms to address additional forms of heterogeneity. Specifically, in \cref{app:extensions}, we show how our algorithms can provably converge under minimal assumptions when the network topology is changing dynamically (\cref{app:extensions-dynamic}) and when the corresponding topology is directed (\cref{app:extensions-directed}). These scenarios are pertinent to real-world sensor networks since sensors and communications can fail, corresponding to dynamically changing networks with asymmetries (cf. \cite{touri2009distributed}). Moreover, to balance the trade-offs between accuracy and privacy, the agents can resort to heterogeneous privacy budgets $\{ \vec \eps_i \}_{i \in [n]}$ and improve their collective estimation performance while maintaining a minimum privacy protection (capping the individual privacy budgets at $\eps_{i,\max}$). The possibility to accommodate heterogeneous budgets in the local DP setting leads to interesting design choices for improving the collective learning performance, e.g., using personalized DP methods \cite{acharya2024personalized,jorgensen2015conservative}. In \cref{app:extensions-heterogenous}, we provide both centralized and decentralized schemes to allocate privacy budgets to optimize their collective accuracy subject to individual privacy budget caps and test their performance on the German Households dataset (cf. Supplementary \cref{supp:fig:mse_plot_heterogeneous}).

\begin{figure}[t]
    \centering
    \subfigure[\footnotesize Ours]{\includegraphics[width=0.33\textwidth]{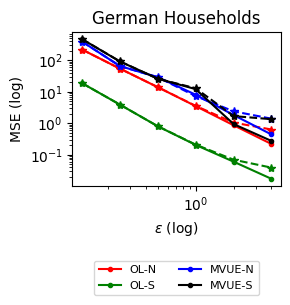}}
    \subfigure[\footnotesize \cite{rizk2023enforcing}]{\includegraphics[width=0.33\textwidth]{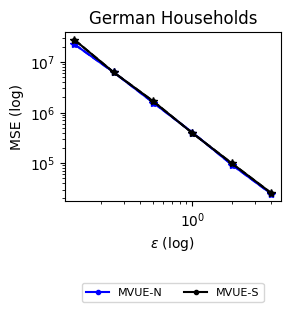}}
    \subfigure[\footnotesize Ours]{\includegraphics[width=0.33\textwidth]{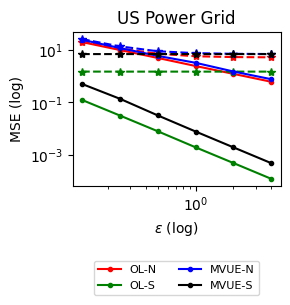}}
    \subfigure[\footnotesize \cite{rizk2023enforcing}]{\includegraphics[width=0.33\textwidth]{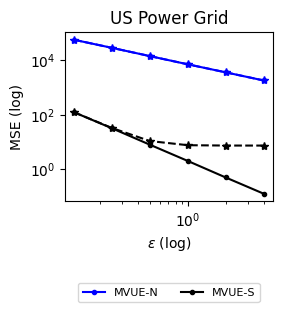}}
    \caption{MSE plots vs. varying privacy budget $\eps$ for the German Households dataset and the US Power Grid Dataset. We compare with the first-order method of \cite{rizk2023enforcing} with a learning rate of $\eta = 0.001$ (see \cref{app:rizk}). The solid lines represent the CoP, and the dashed lines represent the Total Error.}
    \label{fig:mse_plot}
\end{figure}

\ppar{Conclusion.} Our paper focuses on distributed estimation and learning in a networked environment subject to privacy constraints. The aim is to estimate the statistical properties of unknown random variables based on observed data. Our aggregation methods aim to combine the observed data efficiently without requiring explicit coordination beyond the local neighborhood of each agent. This allows for estimating a complete sufficient statistic using either offline or online signals provided to the agents. To preserve privacy, agents add noise to their estimates, adhering to a differential privacy budget ($\eps$-DP)  to safeguard the privacy of either their signals (signal DP) or their signals and network neighborhoods (Network DP). Our algorithms employ linear aggregation schemes that combine the observations of all agents while incorporating the added noise, either online or offline. We prove that the estimation error bounds depend on two terms: the first term corresponds to the error incurred due to the aggregation scheme (which we call Cost of Decentralization) and can be controlled by the mixing rate of the doubly-stochastic adjacency weights, and the second term corresponds to the error due to the DP noising (which we call Cost of Privacy). We prove that under all cases (see also \cref{sec:contribution}), the noise distributions that minimize the convergence rate correspond to the Laplace distributions with parameters that depend on the (local or global) signal sensitivities, the network structure, and the differential privacy budget $\eps$. Finally, we test our algorithms and validate our theory in numerical experiments. 

When sensitivities are locally bounded, signal DP can be achieved efficiently with a graceful accuracy loss over a decreasing privacy budget. This is facilitated by the post-processing immunity of DP \cite[Proposition 2.1]{dwork2014algorithmic} that no future leaks are possible after adequate noising of the private signals and indicates the resilience of linear aggregation schemes to DP noising. However, achieving network DP with noising of estimates is significantly more challenging, and while individual noisy estimates are protected against a one-time attack, network information can still leak over time across multiple estimates. The composition property \cite[Chapter 3]{dwork2014algorithmic} implies that we can protect the network neighborhoods at $\eps$-DP level against an adversary who eavesdrops $k$ times by protecting individual estimates at $\eps/k$-DP level. Such protection can be challenging if an adversary can eavesdrop on the estimates for a long time or has simultaneous access to estimates of multiple agents. In such cases, a fast convergence rate (using the fastest mixing weights) can limit communications and help agents maintain privacy without completely sacrificing the accuracy of their estimates.

\bibliographystyle{apalike}
\spacingset{1}
% Bibliography
% \bibliographystyle{ACM-Reference-Format}
% \footnotesize{{\setstretch{0.8}%\renewcommand{\baselinestretch}{0.25} 
% \setlength{\bibsep}{4pt plus 0.5ex}
{\footnotesize \bibliography{references}}
% }}

\subsection*{Acknowledgements}
{\small M.P. was partially supported by a LinkedIn Ph.D. Fellowship, an Onassis Fellowship (ID: F ZT 056-1/2023-2024), and grants from the A.G. Leventis Foundation and the Gerondelis Foundation. M.A.R. was partially supported by NSF SaTC-2318844. The authors would like to thank the seminar participants at Rutgers Business School, Jalaj Upadhyay, Saeed Sharifi-Malvajerdi, Jon Kleinberg, Kate Donahue, and Vasilis Charisopoulos for their valuable discussions and feedback.}

\subsection*{Data Availability Statement}
{\small 
The data that support the findings of this study are openly available in the following repositories: 

\begin{enumerate}
    \item \textbf{GEM House openData.} URL: \url{https://dx.doi.org/10.21227/4821-vf03} (Accessed at 1-6-2023), Reference: \cite{milojkovic2018gem}.
    \item \textbf{US Power Grid Network.} URL: \url{https://toreopsahl.com/datasets/#uspowergrid} (Accessed at 1-6-2023), Reference: \cite{watts1998collective}.
\end{enumerate}

\noindent Our code and data can be found at: \url{https://github.com/papachristoumarios/dp-distributed-estimation}.
%(online repository link removed for anonymity).%
}
\subsection*{Biographical Sketches}

\ppar{Marios Papachristou.} Marios Papachristou is a 4th year PhD student at the Department of Computer Science at Cornell University and is advised by Jon Kleinberg. His research interests span the theoretical and applied aspects of social and information networks, exploring their roles within large-scale social and information systems and understanding their wider societal implications. His research is supported by the Onassis Scholarship and has been supported in the past by a LinkedIn Ph.D. Fellowship, a grant from the A.G. Leventis Foundation, a grant from the Gerondelis Foundation, and a Cornell University Fellowship.

\medskip

 \ppar{M. Amin Rahimian.}  Amin Rahimian has been an assistant professor of industrial engineering at the University of Pittsburgh since 2020, where he leads the sociotechnical systems research lab. Prior to that, he was a postdoc with joint appointments at MIT Institute for Data, Systems, and Society (IDSS) and MIT Sloan School of Management. He received his PhD in Electrical and Systems Engineering from the University of Pennsylvania, and Master’s in Statistics from Wharton School. Broadly speaking his works are at the intersection of networks, data, and decision sciences, and have been published in the Proceedings of the National Academy of Sciences, Nature Human Behaviour, Nature Communications, and the Operations Research journal among others. His research interests are in applied probability, applied statistics, algorithms, decision and game theory, with applications ranging from online social networks, public health, and e-commerce to modern civilian cyberinfrastructure and future warfare. His research is currently supported by NSF, CDC, and the Department of the Army. 

\newpage

\appendix
\break
  \pagenumbering{arabic}
  \renewcommand{\thepage}{\thesection-\arabic{page}}

\setcounter{figure}{0}
\renewcommand{\figurename}{Supplementary Figure}

\setcounter{equation}{0}
\renewcommand{\theequation}{\thesection.\arabic{equation}}

% \noindent{\Huge \bf Supplementary Material}

\section*{\Huge Supplementary Material}

\section{Proofs} \label{app:proofs}

\subsection{Proof of \texorpdfstring{\cref{theorem:cop_min_var_unbiased_estimation_signal}}{theorem-cop-online-learning-dp-signal}}\label{proof:theorem:cop_min_var_unbiased_estimation_signal}

\begin{proof}

Since $A$ is real and symmetric, we can do an eigen decomposition of $A$ as $A = Q \Lambda Q^T$ where $Q$ is an orthonormal eigenvector matrix, and $\Lambda$ is the diagonal eigenvalue matrix. We have that {{$\| \ovec \nu_t - \ovec \mu_t \|_2^2 = \| Q \Lambda^t Q^T \ovec d \|_2^2 = \sum_{i = 1}^n \lambda_i^{2t}(A) (\ovec q_i^T \ovec d)^2$}}. Note that {{$\ev {} {(\ovec q_i^T \ovec d)^2} = \ev {} {\sum_{j = 1}^n \vec q_{ij}^2 \vec d_j^2 + \sum_{1 \le j < k \le n} \vec q_{ij} \vec q_{ik} \vec d_j \vec d_k} = \ev {} {\sum_{j = 1}^n \vec q_{ij}^2 \vec d_j^2} = \sum_{i = 1}^n \vec q_{ij}^2 \var {} {\vec d_i}$}}. Therefore {{$\ev {} {\| \ovec \nu_t - \ovec \mu_t \|_2^2} = \sum_{j = 1}^n \var {} {\vec d_j} \sum_{i = 1}^n \lambda_i^{2t}(A) \vec q_{ij}^2$}}. Since for all $2 \le i \le n$ we have that $|\lambda_i(A)| \le \beta^\star$, and by using Jensen's inequality we get that 
{{
\begin{equation*}
\begin{split}
	\ev {} {\| \ovec \nu_t - \ovec \mu_t \|_2} & \le \sqrt {\sum_{i, j 
 = 1}^n \var {} {\vec d_j} \lambda_i^{2t}(A) \vec q_{ij}^2} \le \sum_{i, j = 1}^n \sqrt {\var {} {\vec d_j}} |\lambda_i^t(A)| |\vec q_{ij}| \\
 & \le \sum_{j = 1}^n \sqrt {\var {} {\vec d_j}} |\vec q_{1j}| + (\beta^\star)^t \sum_{i, j = 1}^n \sqrt {\var {} {\vec d_j}} |\vec q_{ij}| \\
 & \le \sqrt {\sum_{j = 1}^n \var {} {\vec d_j}} \left (1 + \sqrt {n - 1} (\beta^\star)^t \right ).
\end{split}	
\end{equation*}}}

To minimize the upper bound on the MSE for each agent $j$, it suffices to minimize the variance of $\vec d_j \sim \cD_j$, subject to differential privacy constraints. We assume that the PDF of $\cD_j$ -- denoted by $p_{\vec d_j}(\cdot) \in \Delta(\Rbb)$ -- is differentiable everywhere in $\Rbb$. The differential privacy constraint is equivalent to 
{{
\begin{align*}
    \left | \frac {d} {d \vec s_j} \log \Pr [\psi_{\cM^S_j}(\vec s_j) = t] \right | & \le \eps \iff \\
    \left | \frac {d} {d \vec s_j} \log p_{\vec d_j}(t - \xi(\vec s_j)) \right | & \le \eps \iff \\
    \left | \frac {d} {d \vec s_j} p_{\vec d_j}(t - \xi(\vec s)) \right | & \le \eps p_{\vec d_j}(t - \xi(\vec s_j)),
\end{align*}}} for all $t \in \Rbb$ and $\vec s_j \in \cS$. Letting $u = t - \xi(\vec s_j)$ we get that, in order to satisfy $\eps$-DP,
{{
\begin{align*}
    \left | \frac {d p_{\vec d_j}(u)} {d u} \right | & \le \frac {\eps} {\Delta} p_{\vec d_j}(u),
\end{align*}}} where $\Delta = \max_{\vec s \in \cS} \left | \frac {d \xi(\vec s_j)} {d \vec s_j} \right |$ is the global sensitivity of $\xi$. We have that 
{{
\begin{align}
    \min_{p_{\vec d_j}(\cdot) \in \Delta(\Rbb)} \quad & \ev {\vec d_j \sim \cD_j} {\vec d_j^2} = \int_{\Rbb} t^2 p_{\vec d_j}(t) dt  \nonumber \\ 
    \text{s.t.} \quad & \int_{\Rbb} p_{\vec d_j}(t) dt = 1, \nonumber \\
    & |p_{\vec d_j}'(t)| \le \frac {\eps} {\Delta} p_{\vec d_j}(t), & \quad \forall t \in \Rbb. \nonumber
\end{align}}}

From Theorem 6 of \cite{koufogiannis2015optimality} we get that the optimal solution to the above problem is the Laplace distribution with scale $\lambda_j = \Delta / \eps$, 
{{
\begin{align} \nonumber
    p_{\vec d_j}(t) = \frac {\eps} {2 \Delta} \exp \left ( - \frac {\eps} {\Delta} |t| \right ), \quad \forall t \in \Rbb.
\end{align}}}
To derive the upper bound on the error note that by Theorem 3 of \cite{rahimian2016distributed} we have that {{$\ev {} {\| \ovec \mu_t - \mathds 1 \hat {\vec m_\theta} \|_2} \le \sqrt {n(n - 1)} (\beta^\star)^{t} M_n$}}, and also {{$\ev {} {\| \ovec \mu_t - \ovec \nu_t \|_2} \le\sqrt {\sum_{j = 1}^n \var {} {\vec d_j}}  (1 +  \sqrt {n - 1} (\beta^\star)^t)$}}. Applying the triangle inequality yields the final result, i.e., 

{{
\begin{align} \label{supp:eq:mvue_total_bound}
    \ev {} {\| \ovec \nu_t - \mathds 1 \hat {\vec m_\theta} \|_2} \le \sqrt {n(n - 1)} (\beta^\star)^{t} M_n + \sqrt {\sum_{j = 1}^n \var {} {\vec d_j}}  (1 +  \sqrt {n - 1} (\beta^\star)^t).
\end{align}

}}

Using the optimal distributions {{$\cD_i^\star = \mathrm {Lap}(\Delta  / \eps)$}} in \eqref{eq:MVUE-signal-DP-bound} gives the claimed upper bound on  {{${\mathrm{TE}}( \cM^S)$}}. 
\end{proof}

\subsection{Proof of \texorpdfstring{\cref{corollary:cop_min_var_unbiased_estimation_network}}{corollary-cop-min-var-unbiased-estimation-network} (DP preservation across time)} \label{proof:corollary:cop_min_var_unbiased_estimation_network} 

To derive the DP guarantee for the MVUE for round $t$, we will do an induction. Specifically, we want to prove that for all $\ovec x = (\vec x_1, \dots, \vec x_n) \in \Rbb^n$ we have for all $i \in [n]$, 
{
\begin{align*}
    \left | \log \left ( \frac {\Pr [\vec \nu_{i, t} = \vec x_i]} {\Pr [\vec \nu_{i, t}' = \vec x_i]} \right ) \right | \le \eps, 
\end{align*}}for all adjacent pairs of signals and beliefs, i.e., $\left \| (\vec s_i, \{ \vec \nu_{j, t- 
 1 } \}_{j \in \cN_i} ) - (\vec s_i', \{ \vec \nu_{j, t- 
 1 }' \}_{j \in \cN_i} ) \right \|_1 \le 1$. We proceed with the induction as follows:

\begin{itemize}
    \item For $t = 1$, the result is held by the construction of the noise and the definition of DP. 
    \item For time $t \in \Nbb$, we assume that $\left | \log \left ( \frac {\Pr [\vec \nu_{i, t} = \vec x_i]} {\Pr [\vec \nu_{i, t}' = \vec x_i]} \right ) \right | \le \eps$ for all $\ovec x = (\vec x_1, \dots, \vec x_n) \in \Rbb^n$. 
    \item For time $t + 1$, we have that for all $i \in [n]$, 

    {
    \begin{align*}
         \left | \log \left ( \frac {\Pr [\vec \nu_{i, t + 1} = \vec x_i]} {\Pr [\vec \nu_{i, t + 1}' = \vec x_i]} \right ) \right | = \left | \log \left ( \frac {\Pr [\vec \nu_{i, t} = (A^{-1} \ovec x)_i]} {\Pr [\vec \nu_{i, t}' = (A^{-1} \ovec x)_i]} \right ) \right | \le \eps.
    \end{align*}}

    which holds by applying the definition of the MVUE update, the fact that $A$ is non-singular, and the inductive hypothesis for $t$. 
    
\end{itemize}

\subsection{Proof of \texorpdfstring{\cref{theorem:cop_online_learning_dp_signal}}{theorem-cop-online-learning-dp-signal}}\label{proof:theorem:cop_online_learning_dp_signal}

\begin{proof}
Similarly to \cref{theorem:cop_min_var_unbiased_estimation_signal}, we decompose $A$ as $A = Q \Lambda Q^T$ and get that 
{{
\begin{equation*}
\begin{split}
	\| \ovec \nu_t - \ovec \mu_t \|_2^2 & = \left \| \frac 1 tQ \sum_{\tau = 0}^{t - 1} \Lambda^\tau \ovec d_{t - \tau} Q^T \right \|_2^2 \\
	& = \frac 1 {t^2} \sum_{i = 1}^n \sum_{\tau = 0}^{t - 1} \lambda_i^{2 \tau} (A) (\ovec q_i^T \ovec d_{t - \tau})^2. 
\end{split}	
\end{equation*}}}

We take expectations and apply Cauchy-Schwarz to get 
{{
\begin{align} \label{supp:eq:error_bound}
    \ev {} {\left \| \ovec \nu_t - \ovec \mu_t \right \|_2^2} & \le \frac 1 {t^2} \sum_{i = 1}^n \sum_{\tau = 0}^{t - 1} \lambda^{2 \tau}_i(A) \ev {} {\left \| \ovec d_{t - \tau} \right \|_2^2} \\
    & \le \frac {1} {t^2} \left (1 + \sum_{i = 2}^n \sum_{\tau = 0}^{t - 1} \lambda_i^{2\tau}(A) \right ) \left ( \sum_{j = 1}^n \sum_{\tau = 0}^{t - 1} \var {} {\vec d_{j, t - \tau}} \right ) \nonumber \\
    & \le \frac {1} {t^2} \left ( 1 + (n - 1) \sum_{\tau = 0}^{t - 1} (\beta^\star)^{2\tau} \right ) \left ( \sum_{j = 1}^n \sum_{\tau = 0}^{t - 1} \var {} {\vec d_{j, t - \tau}} \right ) \nonumber \\
    & \le \frac {1} {t^2} \left ( 1 + \frac {n - 1} {1 - (\beta^\star)^2} \right ) \left ( \sum_{j = 1}^n \sum_{\tau = 0}^{t - 1} \var {} {\vec d_{j, t - \tau}} \right ) \nonumber.
\end{align}}}

By Jensen's inequality, we get that 

{{\begin{align*}
    \ev {} {\left \| \ovec \nu_t - \ovec \mu_t \right \|_2} \le \frac 1 t \left ( 1 + \sqrt {\frac {n - 1} {1 - (\beta^\star)^2}} \right ) \sqrt {\sum_{j = 1}^n \sum_{\tau = 0}^{t - 1} \var {} {\vec d_{j, t - \tau}}}.
\end{align*}}}

Also note that the dynamics of $\ovec q_t = \ovec \mu_t - \mathds 1 \vec m_{\theta}$ obey

{{\begin{align} \nonumber
    \ovec q_t = \frac {t - 1} {t} A \ovec q_{t - 1} + \frac 1 t \left ( \ovec \xi_t - \mathds 1 \vec m_{\theta} \right ).
\end{align}}}

By following the same analysis as \cref{supp:eq:error_bound} we get that 

{{\begin{align} \nonumber
     \ev {} {\left \| \ovec \nu_t - \ovec \mu_t \right \|_2} \le \sqrt {\frac n t} \left ( 1 + \sqrt {\frac {n - 1} {1 - (\beta^\star)^2}} \right ) \sqrt {\var {} {\xi(\vec s)}},
\end{align}}} and then, the triangle inequality yields

{{\begin{align*}
   \ev {} {\left \| \ovec \nu_t - \mathds 1 \vec m_{\theta} \right \|_2} \le \frac 1 t \left ( 1 + \sqrt {\frac {n - 1} {1 - (\beta^\star)^2}} \right ) \left ( \sqrt{nt \var {} {\xi(\vec s)}} + \sqrt {\sum_{j = 1}^n \sum_{\tau = 0}^{t - 1} \var {} {\vec d_{j, t - \tau}}} \right ).
\end{align*}}}

To optimize the upper bound of \cref{supp:eq:error_bound}, for every index $0 \le \tau \le t - 1$ and agent $j \in [n]$ we need to find the zero mean distribution $\cD_{j, \tau + 1}$ with minimum variance subject to differential privacy constraints. We follow the same methodology as \cref{theorem:cop_min_var_unbiased_estimation_signal} and arrive at the optimization problem 

{{\begin{align} \nonumber
    \min_{p_{\vec d_{j, \tau + 1}}(\cdot) \in \Delta(\Rbb)} \quad & \ev {\vec d_{j, \tau + 1} \sim \cD_j} {\vec d_{j, \tau + 1}^2} = \int_{\Rbb} u^2 p_{\vec d_{j, \tau + 1}}(u) du  \\ 
    \text{s.t.} \quad & \int_{\Rbb} p_{\vec d_{j, \tau + 1}}(u) du = 1, \nonumber \\
    & |p_{\vec d_{j, \tau + 1}}'(u)| \le \frac {\eps} {\Delta} p_{\vec d_{j, \tau + 1}}(u), & \quad \forall u \in \Rbb. \nonumber
\end{align}}}

The optimal distribution is derived identically to \cref{theorem:cop_min_var_unbiased_estimation_signal}, and equals $\cD_{j, \tau + 1}^\star =     \mathrm {Lap} \left ( \frac {\Delta} {\eps} \right )$ for all $j \in [n]$ and $0 \le \tau \le t - 1$. 

\end{proof}

\subsection{Proof of \texorpdfstring{\cref{theorem:cop_online_learning_dp_network}}{theorem-cop-online-learning-dp-network}}\label{proof:theorem:cop_online_learning_dp_network}

\begin{proof}
     Let $C(t) = B(t) - \frac 1 t I$, and let $\Phi(t) = \prod_{\tau = 0}^{t - 1} C(\tau)$. Note that $C(t)$ can be written as $t C(t) = (t - 2) I + A$, and we can infer that the eigenvalues of $C(t)$ satisfy $\lambda_i(C(t)) = 1 + \frac {\lambda_i(A) - 2} {t}$, and that $\left \{ C(\tau) \right \}_{\tau \in [t]}$ and $A$ have the same eigenvectors. Therefore,

    {{\begin{align*}
        \lambda_i(\Phi(t)) = \prod_{\tau = 0}^{t - 1} \lambda_i(C(\tau)) \le \exp \left ( \sum_{\tau = 1}^t \frac { \lambda_i(A) - 2} {\tau} \right ) \le t^{\lambda_i(A) - 2}.
    \end{align*}}}

       Similarly to \cref{theorem:cop_online_learning_dp_signal} we have that 

    {{\begin{align*}
        \ev {} {\left \| \ovec \nu_t - \ovec \mu_t \right \|_2^2} & = \frac {1} {t^2} \sum_{i = 1}^n \sum_{\tau = 0}^{t - 1} \lambda_i(\Phi(\tau))^2 \ev {} {\left ( \ovec q_i^T \ovec d_{t - \tau} \right )^2} \\
        & \le \frac {1} {t^2} \left ( \sum_{i = 1}^n \sum_{\tau = 0}^{t - 1} \lambda_i(\Phi(\tau))^2 \right ) \left (  \sum_{\tau = 0}^{t - 1}\ev {} {\left \| \ovec d_{t - \tau} \right \|_2^2} \right ) \\ 
        & \le \frac {1} {t^2} \left ( \int_{1}^{t} \tau^{2 \lambda_i(A) - 4} d \tau  \right )  \sum_{\tau = 0}^{t - 1}\ev {} {\left \| \ovec d_{t - \tau} \right \|_2^2} \\
        & \le \frac {1} {t^2} \left ( \sum_{i = 1}^n \int_{1}^{t} \tau^{2\lambda_i(A) - 4} d \tau  \right ) \left ( \sum_{i = 1}^n \sum_{\tau = 0}^{t - 1} \ev {} {\left \| \ovec d_{t - \tau} \right \|_2^2} \right ) \\
        & \le \frac {1} {t^2} \left ( \sum_{i = 1}^n \frac {1} {3 - 2 \lambda_i(A)}  \right ) \sum_{\tau = 0}^{t - 1}\ev {} {\left \| \ovec d_{t - \tau} \right \|_2^2} \\
        & \le \frac {1} {t^2} \left ( 1 + \frac {n - 1} {3 - 2 \beta^\star}  \right ) \sum_{\tau = 0}^{t - 1}\ev {} {\left \| \ovec d_{t - \tau} \right \|_2^2} \\
        & \le \frac {1} {t^2} \left ( 1 + \frac {n - 1} {3 - 2 \beta^\star}  \right ) \left ( \sum_{i = 1}^n \sum_{\tau = 0}^{t - 1} \var {} {\vec d_{i, t - \tau}} \right ).
    \end{align*}}}

    Applying Jensen's inequality, we get that 
    {{\begin{align*}
        \ev {} {\left \| \ovec \nu_t - \ovec \mu_t \right \|_2} \le \frac 1 t \sqrt{ \sum_{i = 1}^n \sum_{\tau = 0}^{t - 1} \var {} {\vec d_{i, t - \tau}}}   \left ( 1 + \sqrt {\frac {n - 1} {3 - 2 \beta^\star}}  \right ).
    \end{align*}}}

    Similarly, by considering the dynamics of $\ovec q_t = \ovec \mu_t - \mathds 1 \vec m_\theta$ we get that {{$$\ev {} {\left \| \ovec \mu_t - \mathds 1 \vec m_\theta \right \|_2} \le \sqrt {\frac n t} \sqrt {\var {} {\xi(\vec s)}}   \left ( 1 + \sqrt {\frac {n - 1} {3 - 2 \beta^\star}}  \right ).$$}} The triangle inequality yields the final error bound.

    To derive the optimal distributions, note that at each round $t$, the optimal action for agent $i$ is to minimize $\var {} {\vec d_{i, t}}$ subject to DP constraints. By following the analysis similar to \cref{theorem:cop_min_var_unbiased_estimation_signal}, we deduce that the optimal noise to add is Laplace with parameter ${\max \{ \max_{j \in \cN_i} a_{ij}, \Delta \}}/ {\eps}$.

\end{proof}

\section{Algorithm of \texorpdfstring{\cite{rizk2023enforcing}}{rizk2023enforcing}} \label{app:rizk}

We adapt the framework of \cite{rizk2023enforcing} to our problem, for which the identification of the MVUE $\hat {\vec m_{\theta}}$ can be formulated as
{{
\begin{align} \nonumber
    \hat {\vec m_{\theta}} = \argmin_{\vec m \in \Rbb} \frac {1} {2n} \sum_{i = 1}^n \underbrace {\left ( \vec m - \xi(\vec s_i) \right )^2}_{J_i(\vec m)}.
\end{align}}}

The private dynamics for updating the beliefs $\ovec \nu_t$ can be found by simplifying the consensus algorithm given in Equations (24)-(26) of \cite{rizk2023enforcing}:

{{
\begin{align} \nonumber
    \vec \nu_{i, t} & =  a_{ii} (\vec \nu_{i, t - 1}+ \vec g_{ii, t}) + \sum_{j \in \cN_i} a_{ij} \left ( \vec \nu_{j, t - 1} + \vec g_{ij, t} \right ) + \vec d_{i, t} - \eta (\vec \nu_{i, t - 1} - \xi(\vec s_i)). 
\end{align}}} 

Here, $\eta$ is the learning rate, $\vec d_{i, t}$ is noise used to protect the private signal, and $g_{ij, t}$,  $\left \{ \vec g_{ij, t} \right \}_{j \in \cN_i}$ are noise terms used to protect own and the neighboring beliefs. As a first difference, we observe that \cite{rizk2023enforcing} uses $(n + m)T$ noise variables, whereas our method uses just $n$ noise variables, which makes our method easier to implement for the MVUE task. It is easy to observe that these dynamics converge slower than our dynamics for two reasons: (i) the privacy protections are added separately for the signal and each neighboring belief, and (ii) the beliefs are always using information from the signals since the method is first-order, thus requiring noise to be added at each iteration. 

For this reason, the authors consider graph-homomorphic noise, i.e., noise of the form $\vec g_{ij, t} = \vec q_{i, t}$ for all $j \neq i$ and $\vec g_{ii, t} = - \frac {1 - a_{ii}} {a_{ii}} \vec q_{i, t}$ where $\vec q_{i, t}$ are noise variables. Rewriting the dynamics in this form, we get the following update: 

{{
\begin{align} \nonumber
    \vec \nu_{i, t} & = (a_{ii} - \eta) \vec \nu_{i, t - 1} + \sum_{j \in \cN_i} a_{ij}  \vec \nu_{j, t - 1} + \eta  \xi(\vec s_i) + \vec d_{i, t}.
\end{align}}}

Given a privacy budget $\eps$, in order to make a fair comparison with our algorithm, the noise variable should be chosen as $\vec d_{i, t} \sim \mathrm{Lap} \left ( \frac {\eta T S^\star_{\xi, \gamma} (\vec s_i)} {\eps} \right )$ for Signal DP, and $\vec d_{i, t} \sim \mathrm{Lap} \left ( \frac {T \max \{ \max_{j \neq i} a_{ij}, \eta S^\star_{\xi, \gamma} (\vec s_i) \}} {\eps} \right )$ for Network DP. Here, since the per-agent privacy budget is $\eps$, and the noise is added $T$ times at each iteration, the initial budget needs to be divided by $T$. 

We run the same experiments as in \cref{sec:experiments} with the method of \cite{rizk2023enforcing} and compare with our method using the same values of the privacy budget $\eps$ and a learning rate $\eta = 0.001$ to get the results in \cref{fig:mse_plot}. 

% Example sample paths of the algorithm are shown in \cref{fig:mvue_rizk}.

\section{Extensions to Dynamic and Directed Networks and Heterogeneous Privacy Budgets } \label{app:extensions}

\subsection{Dynamic Networks} \label{app:extensions-dynamic}

In this problem, the agents observe a sequence of dynamic networks $\{ G(t) \}_{t \in \Nbb}$, for example, due to corrupted links, noisy communications, other agents choosing not to share their measurements, power failures etc. These dynamic networks correspond to a sequence of doubly stochastic matrices $\{ A(t) \}_{t \in \Nbb}$. A  choice for the weights are the modified Metropolis-Hastings (MH) weights:

{{
\begin{align} \label{supp:eq:mh_weights}
    a_{ij}(t) = \begin{cases}
        \frac {1} {2 \max \{ \deg_t(i), \deg_t(j) \}}, & j \neq i \\
        1 - \sum_{j \in \cN_i} a_{ij}(t), & j = i
    \end{cases}.
\end{align}}}

This is a natural choice of weights since they can be easily and efficiently computed by the agents in a distributed manner (e.g., in a sensor network) from the knowledge of own and neighboring degrees, thus requiring minimal memory to be stored for each agent. Below we will show that the MVUE and the OL algorithms converge under minimal assumptions for the time-varying MH weights.

\ppar{MVUE.} We can prove that if $G_t$ contain no isolated nodes, the beliefs would converge to $\hat {\vec m_\theta}$ as a direct consequence of \cite[Theorem 1.4]{chazelle2011total}, i.e.

\begin{proposition}
    If $G_t$ contains no isolated nodes for all $t \in \Nbb$, for accuracy $0 < \rho < 1 / 2$, the dynamics with the modified MH weights of \cref{supp:eq:mh_weights} will converge to the MVUE, $\hat {\vec m_\theta}$, in $$t = \min \left \{ \frac {2^{O(n)} (M_n + \Delta / \eps)} \rho , \left ( \log \frac {M_n + \Delta / \eps} \rho \right )^{n - 1} 2^{n^2 - O(n)} \right \}$$ steps.    
\end{proposition}

\begin{proof}
    The proof follows from applying Theorem 1.4 of \cite{chazelle2011total}, since $\max_{j \neq i} a_{ij} \le 1 / 2$, the graph is undirected, and the diameter of the points is at most $2 (M_n + \Delta / \eps)$. 
\end{proof}

Note that the above convergence rate is exponentially worse than the $$t = O \left ( \frac {\log \left ( {n (M_n + \Delta / \eps)} / \rho \right )} {\log (1 / \beta^\star)} \right )$$ convergence time that we obtain for static networks.

\ppar{Online Learning.} For the online learning regime, we can follow the approach presented in \cite{touri2009distributed}. For this, we consider the following update rule: 

\begin{align} \label{supp:eq:time_varying}
    \ovec \nu_t = \frac {t - 1} {t} \ovec \nu_{t - 1} + \frac 1 t A(t) \ovec \nu_{t - 1} + \frac 1 t \left ( \ovec \xi_t + \ovec d_t \right ).
\end{align}

We show that under reasonable assumptions on the graph sequence, the above algorithm converges to $\one \vec m_\theta$ almost surely. 

\begin{proposition}
     If $G_t$ contains no isolated nodes for all $t \in \Nbb$, and there exists an integer $B$ such that $\bigcup_{\tau = tB}^{t (B + 1) - 1} G_\tau$ is strongly connected for every $t \in \Nbb$, then the dynamics of \cref{supp:eq:time_varying} with the modified MH weights of \cref{supp:eq:mh_weights} converge to $\one \vec m_\theta$ almost surely. 
\end{proposition}

\begin{proof}

The dynamics have the form of the dynamics of \cite{touri2009distributed}. We will show the result by showing that the assumptions of \cite{touri2009distributed} hold. Specifically, all non-zero elements of $A(t)$ have value at least $1 / (2n)$, $a_{ii}(t) \ge 1/2$ for all $i \in [n], t \in \Nbb$ by our hypothesis, there exists an integer $B$ such that $\bigcup_{\tau = tB}^{t (B + 1) - 1} G_\tau$ is strongly connected for every $t \in \Nbb$ by our hypothesis, $\sum_{t = 1}^\infty 1 / t = \infty$, and $\sum_{t = 1}^\infty 1 / t^2 < \infty$. Thus, by Proposition 2 of \cite{touri2009distributed}, the beliefs converge almost surely to $\one \vec m_\theta$ as $t \to \infty$. Furthermore, we can show that the dependence on $t$ becomes slower, i.e., from $t^{-1/2}$ when the network is static to $t^{-1/n^3}$ when the network is dynamic. 

\end{proof}

\subsection{Directed Networks} \label{app:extensions-directed}

We can also consider extensions of our results to directed graphs, i.e., when $a_{ij} \neq a_{ji}$ in general. Such asymmetries can arise due to systemic heterogeneities, e.g., varying accuracy and reliability of sensors in a sensor network or imbalances caused by influence dynamics in a social network. The adjacency weights in such cases are no longer doubly stochastic, making the MVUE dynamics in \cref{eq:estimateUpdateAvgConcensus-non-private,eq:estimateUpdateAvgConcensus} to converge to $\hat {\vec m_\theta} = \ovec q_1^T \ovec \xi$ where $\ovec q_1$ is the stationary distribution of the Markov chain with transition matrix $A$, satisfying $A \ovec q_1 = \ovec q_1$ and $\ovec q_1^T \one = 1$. Unlike the doubly stochastic case, the stationary distribution is not necessarily uniform, in which case the aggregate converges to a weighted average of $\xi(\vec s_i)$ that is no longer minimum variance. To analyze the convergence rate of the algorithms with asymmetric adjacency weights, we cannot use tools from the theory of doubly-stochastic matrices anymore, and we need to rely on different tools from the analysis of convergence of non-reversible Markov chains. The results of \cite{chatterjee2023spectral} extend the notion of spectral gap and can make this analysis possible. However, these results come with significant technicalities, using substantially different analytical steps that are beyond the scope of our present work. More limited conclusions about the asymptotic rates and finite-time convergence can be asserted in special cases by applying existing results such as \cite{touri2009distributed}. 

\subsection{Heterogeneous Privacy Budgets} \label{app:extensions-heterogenous}

Our algorithms extend easily to the case where each agent has their own privacy budget $\vec \eps_i$, e.g., due to their differing energy consumption levels. In this case, it is easy to show that all of the results can be updated to accommodate heterogeneous $\vec \eps_i$ and the smooth sensitivities $\vec \Delta_i$, as shown in \cref{tab:detailed_bounds_heterogeneous} below.

\begin{table}[ht]
    \centering
    \caption{Total Error Bounds for heterogeneous privacy budgets $\{ \vec \eps_i \}_{i \in [n]}$. The {\color{blue}{blue}} terms are due to privacy constraints (CoP), and the {\color{red}{red}} terms are due to decentralization (CoD). Here $M_n$ and $\beta^\star$ are the same as in \cref{tab:detailed_bounds}, and $\{ \vec \Delta_i \}_{i \in [n]}$ are the smooth sensitivities for each agent which can be set as $\vec \Delta_i = S_{\xi, \gamma}^*(\vec s_i)$ for the case of Signal DP and $\vec \Delta_i = \max \left \{ \max_{j \neq i} a_{ij}, S_{\xi, \gamma}^*(\vec s_i) \right \}$ for the case of Network DP.}
    \footnotesize
    \begin{tabular}{cc}
        \toprule
          Minimum Variance Unbiased Estimation & Online Learning of Expected Values \\
        \toprule
        $O \left (  (\beta^\star)^t {\color{red}{M_n}} + (1 + (\beta^\star)^t){\color{blue}{\sum_{i = 1}^n \frac {\vec \Delta_i} {\vec \eps_i}}} \right )$ & $O \left (  \frac {n} {\sqrt t} {\color{red}{\sqrt {\var {} {\xi(\vec s)}}}} + \frac {1} {\sqrt t} {\color{blue}{\sum_{i = 1}^n \frac {\vec \Delta_i} {\vec \eps_i}}} \right ) $ \\
         % \midrule
         % Network DP & $O \left (  (\beta^\star)^t {\color{red}{M_n}}  + (1 + (\beta^\star)^t){\color{blue}{\sum_{i = 1}^n \frac {\max \{ a, \Delta \}} {\vec \eps_i}}} \right )$ & $O \left (  \frac {n} {\sqrt t} {\color{red}{\sqrt {\var {} {\xi(\vec s)}}}} + \frac {1} {\sqrt t} {\color{blue}{\sum_{i = 1}^n \frac {\max \{a , \Delta \}} {\vec \eps_i}}} \right ) $ \\ 
        
        \bottomrule
    \end{tabular}
    
    \label{tab:detailed_bounds_heterogeneous}
\end{table}

\begin{figure}[t]
    \centering
    \includegraphics[width=0.75\textwidth]{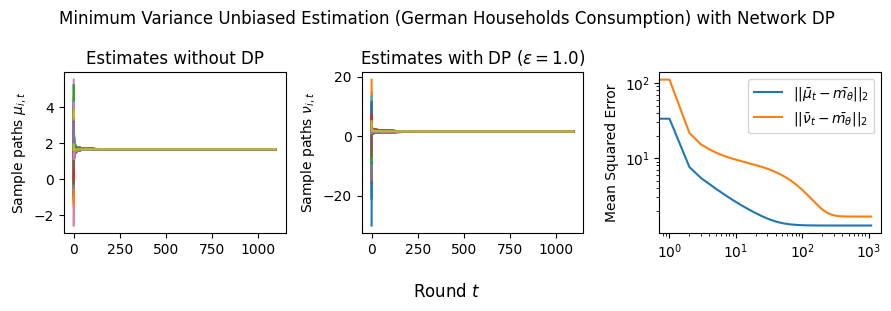}
    \includegraphics[width=\textwidth]{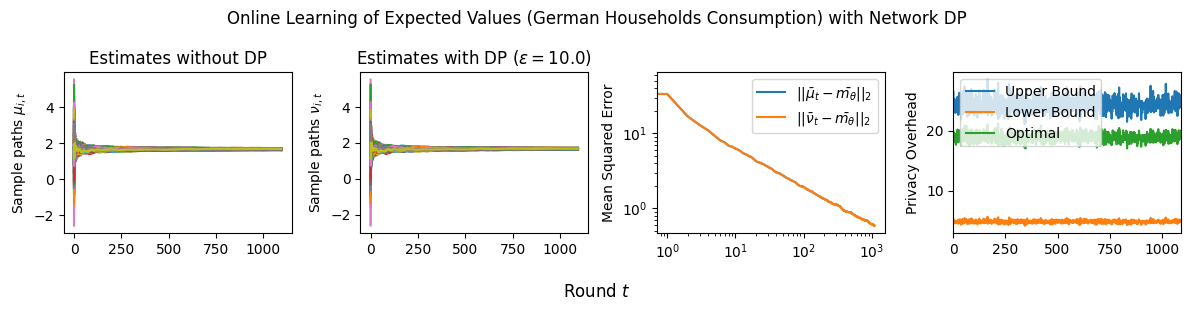}
    \caption{Sample Paths for MVUE and OL for the German Households Dataset with heterogeneous budgets (centralized solution). For the OL case, we plot the optimal privacy overhead $\sum_{i = 1}^n \frac {\vec \Delta_i} {\vec \eps_i^\star}$ which we compare with the lower bound $\sum_{i = 1}^n \frac {\vec \Delta_i} {\eps}$, and the upper bound $\sum_{i = 1}^n \frac {\vec \Delta_i} {\eps_{i, \max}}$.}
    \label{supp:fig:distribute_budget-centralized}
\end{figure}

\begin{figure}[t]
    \centering
    \includegraphics[width=0.75\textwidth]{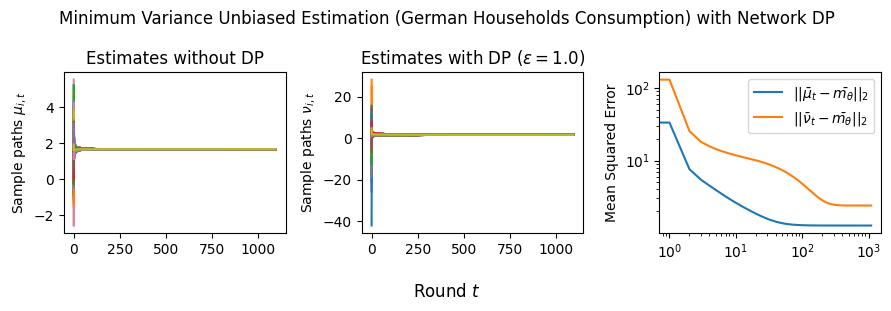}
    \includegraphics[width=\textwidth]{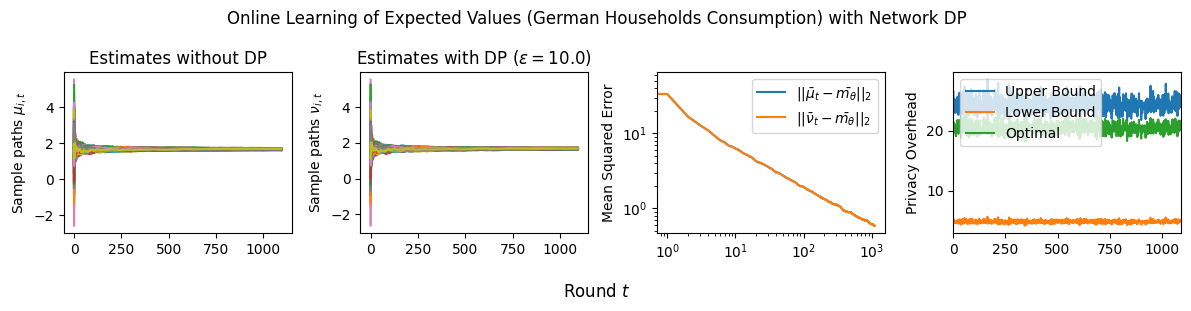}
    \caption{Sample Paths for MVUE and OL for the German Households Dataset with heterogeneous budgets (decentralized solution). For the OL case, we plot the optimal privacy overhead $\sum_{i = 1}^n \frac {\vec \Delta_i} {\vec \eps_i^\star}$ which we compare with the lower bound $\sum_{i = 1}^n \frac {\vec \Delta_i} {\eps}$, and the upper bound $\sum_{i = 1}^n \frac {\vec \Delta_i} {\eps_{i, \max}}$.}
    \label{supp:fig:distribute_budget-decentralized}
\end{figure}

\begin{figure}[t]
    \centering
    \subfigure[\footnotesize Homogeneous]{\includegraphics[width=0.32\textwidth]{figures/mse_plot_germany_consumption.png}}
    \subfigure[\footnotesize Problem \eqref{supp:eq:heterogeneous_eps}]{\includegraphics[width=0.32\textwidth]{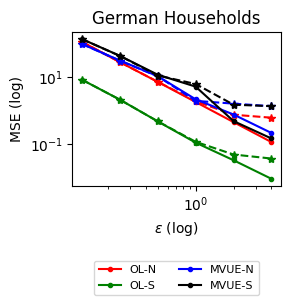}}
    \subfigure[\footnotesize Problem \eqref{supp:eq:heterogeneous_eps_nbr}]{\includegraphics[width=0.32\textwidth]{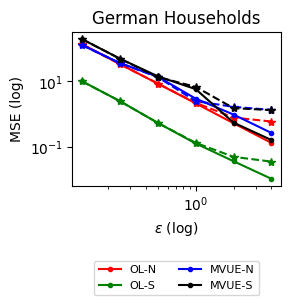}}
    \caption{MSE Plots for the German Households Dataset with heterogeneous privacy budgets. We note that compared to the homogeneous case, using heterogeneous budgets reduces the MSE.}
    \label{supp:fig:mse_plot_heterogeneous}
\end{figure}

The possibility of heterogeneous privacy budgets opens new avenues to explore the allocation of an overall privacy budget ($n\eps$) while respecting individual budgets $\{\eps_{i,\max}\}_{i\in[n]}$. Such a situation may arise, e.g., to limit overall information leakage against an adversary that can eavesdrop on some or all of the communications. Suppose that each agent wants to maintain $\vec \eps_i$-DP and suppose that there is an adversary that eavesdrops on all of the beliefs (this also covers the cases where the adversary has access to a subset $W $ of the $[n]$, such as in the case of a targeted attack to a large part of the network). Note that all of the results presented so far protect against information leakage about a single agent's signals (or their network neighborhoods) when their reports $\nu_{i,t}$ are compromised. However, if the goal is to protect the private signals against an adversary that can eavesdrop simultaneously on all or a subset of agents, then  each individual agent's report can be regarded as a data release about the vector of all signals that needs to be protected at the $n \eps$-DP level, in addition to the individual-level $\vec \eps_{i,\max}$-DP protections required by each agent. If the adversary can eavesdrop on all of the signals, then the resulting mechanism that protects the joint distribution of the signals can be thought of as a mechanism $\Psi^{\cM^S}$ which adds $n$-dimensional noise $\ovec d$ to the sufficient statistics and each dimension $i \in [n]$ of the noise corresponds to $\vec d_{i,t}$, i.e., 
{{
\begin{align} \nonumber
    \Psi^{\cM^S}(\vec s_{1, t}, \dots, \vec s_{n, t}) = \ovec \xi_t + \ovec d_t
\end{align}}
Then, if $\vec \Delta_i$ is the sensitivity for each agent $i$, the noise has PDF:

{{
\begin{align} \nonumber
    p_{\ovec d_t}(\vec u_1, \dots, \vec u_n) = \prod_{i = 1}^n \frac {\vec \eps_i} {2 \vec \Delta_i} \exp \left ( - \frac {\vec \eps_i} {\vec \Delta_i} |\vec u_i| \right ) \propto \exp \left ( - \sum_{i = 1}^n \frac {\vec \eps_i} {\vec \Delta_i} |\vec u_i| \right )
\end{align}}

Now, consider a pair $(\vec s_{1, t}, \dots, \vec s_{n, t}), (\vec s_{1, t}', \dots, \vec s_{n, t}')$ of sets of signals such that $$\left \| (\vec s_{1, t}, \dots, \vec s_{n, t}) - (\vec s_{1, t}', \dots, \vec s_{n, t}') \right \|_1 \le 1.$$ Then, we have that for all $\ovec x \in \Rbb^n$:

{{
\begin{align} \nonumber
    \left | \log \left ( \frac {\Pr [\Psi^{\cM^S_{i, t}}(\vec s_{1, t}, \dots \vec s_{n, t}) = \ovec x]} {\Pr [\Psi^{\cM^S_{i, t}}(\vec s_{1, t}', \dots \vec s_{n, t}') = \ovec x]} \right ) \right | &  = \left | \log \left ( \frac {p_{\ovec d_t}(\ovec \xi_t - \ovec x)} {p_{\ovec d_t} (\ovec \xi_t' - \ovec x)} \right ) \right | \\
    & =  \left | \sum_{i = 1}^n \frac {\vec \eps_i} {\vec \Delta_i} (|\xi(\vec s_{i, t}') - \ovec x| - |\xi(\vec s_{i, t})  - \ovec x|) \right | \nonumber \\
    % & \leq  \left | \sum_{i = 1}^n \frac {\vec \eps_i} {\vec \Delta_i} (\xi(\vec s_{i, t}) - \xi(\vec s_{i, t}')) \right | \nonumber \\
    & \le  \sum_{i = 1}^n \frac {\vec \eps_i} {\vec \Delta_i} |\xi(\vec s_{i, t}) - \xi(\vec s_{i, t}')|  \nonumber\\
    & \le  \sum_{i = 1}^n {\vec \eps_i} \left \|\vec s_{i, t} - \vec s_{i, t}' \right \|_1  \nonumber\\
    & \le \left ( \sum_{i = 1}^n \vec \eps_i \right ) \max_{i \in [n]} \left \|\vec s_{i, t} - \vec s_{i, t}' \right \|_1  \nonumber\\
    & \le \sum_{i = 1}^n \vec \eps_i. \label{supp:eq:heterogenous-privacy-global-provacy-bound}
\end{align}}

Now suppose we want to protect the vector of all beliefs against the eavesdropper at the $n \eps$-DP level (assuming an average privacy budget of $\eps$ per agent for the overall protection of the vector of all private signals). The noise that is added to individual estimates also works to protect the entire vector of all estimates against the eavesdropper and to achieve the latter at the $n\eps$ level, \cref{supp:eq:heterogenous-privacy-global-provacy-bound} indicates that it is sufficient to ensure that $\sum \vec \eps_i \le n \eps$. On the other hand, given $\eps_i$ privacy level at every agent $i$, we also want to minimize the accuracy loss by reducing $\sum_{i=1}^{n}\Delta_i/\eps_i$ while ensuring that individual privacy budgets do not exceed a preset maximum:  $vec \eps_i \leq \vec \eps_{i, \max}$ for all $i$. The subsequent optimization problem to allocate individual privacy budgets can be formulated as follows:

{{
\begin{align} \label{supp:eq:heterogeneous_eps}
    \min_{\vec \eps_1, \dots, \vec \eps_n > 0} \quad & \sum_{i = 1}^n \frac {\vec \Delta_i} {\vec \eps_i} \\
    \text{s.t.} \quad & \sum_{i = 1}^n \vec \eps_i \le n \eps. \nonumber \\
     & \vec \eps_{i} \le \vec \eps_{i, \max}  \quad \forall i \in [n] \nonumber
\end{align}}

Following KKT conditions \cite{Boyd2005a}, \cref{supp:eq:heterogeneous_eps} admits a closed-form solution. The solution indicates that by allowing heterogeneous privacy budgets and taking into account individual smooth sensitivities in the optimal allocation, we can improve the total error. These observations are summarized below:

\begin{proposition}\label{proposition:benefit-of-heterogenity}
    The following hold:

    \begin{enumerate}
        \item If $\sum_{i = 1}^n \vec \eps_{i, \max} \ge n \eps$, then the optimal solution to \cref{supp:eq:heterogeneous_eps} is $$\vec \eps_i^\star = \min \left \{ \vec \eps_{i, \max},  \frac {n \eps \sqrt {\vec \Delta_i}} {\sum_{j \in [n]} \sqrt {\vec \Delta_j}} \right \}$$ for all $i \in [n]$. Moreover, the improvement over $\sum_{i = 1}^n \frac {\vec \Delta_i} {\eps}$ satisfies $\frac {\min_{i \in [n]} \vec \Delta_i} {\max_{i \in [n]} \vec \Delta_i} \le \frac {\sum_{i = 1}^n \frac {\vec \Delta_i} {\vec \eps_i}} {\sum_{i = 1}^n \frac {\vec \Delta_i} {\eps}} \le 1$.
        
        \item If $\sum_{i = 1}^n \vec \eps_{i, \max} < n \eps$ then the optimal solution is $\vec \eps_{i}^\star = \vec \eps_{i, \max}$ for all $i \in [n]$. 
    \end{enumerate}
\end{proposition}

In a large network making individual nodes aware of their allocated budgets based on their smooth sensitivities is difficult to achieve in a central manner. The following formulation of the allocation problem arrives at a sub-optimal solution that satisfies the $n\eps$ global privacy budget constraint by imposing $n$ additional constraints in the local neighborhoods: $a_{ii} \vec \eps_i + \sum_{j \in \cN_i} a_{ij} \vec \eps_j \le \eps, \quad \forall i \in [n]$. The advantage of these constraints is that they can be verified locally, and satisfying them implies the $n\eps$ global budget constraint in \cref{supp:eq:heterogeneous_eps}. The subsequent optimization problem is given in \cref{supp:eq:heterogeneous_eps_nbr}. It allows individuals to learn their allocated budgets in a distributed manner by running distributed gradient descent, which is guaranteed to converge since the problem is convex \cite[Chapter 7]{sayed2014adaptation}.

{{
\begin{align} \label{supp:eq:heterogeneous_eps_nbr}
    \min_{\vec \eps_1, \dots, \vec \eps_n > 0} \quad & \sum_{i = 1}^n \frac {\vec \Delta_i} {\vec \eps_i} \\
    \text{s.t.} \quad & a_{ii} \vec \eps_i + \sum_{j \in \cN_i} a_{ij} \vec \eps_j \le \eps, \quad \forall i \in [n] \nonumber \\
     & \vec \eps_{i} \le \vec \eps_{i, \max},  \quad \forall i \in [n]. \nonumber
\end{align}}

\ppar{Numerical Experiment.} We test our method with the German Households dataset. Specifically, we set a per-agent average budget of $\eps = 1$ for MVUE and $\eps = 10$ for OL. We put a maximum individual budget cap of $\vec \eps_{i, \max} = 10 \eps$ in both cases. We report the sample paths in Supplementary \cref{supp:fig:distribute_budget-centralized,supp:fig:distribute_budget-decentralized}, and observe that the dynamics converge faster compared to the homogeneous case (cf. \cref{fig:sample_paths_min_var_unbiased_estimation_network}). In Supplementary \cref{supp:fig:mse_plot_heterogeneous}, we present an MSE plot where the MSE is plotted as a function of $\eps$, and observe that the algorithm with the heterogeneous thresholds has smaller MSE compared to the homogeneous thresholds. These results confirm our theoretical observations in \cref{proposition:benefit-of-heterogenity}.

\end{document}